\documentclass[11pt]{article}
\usepackage{hyperref}
\usepackage{verbatim}
\usepackage{amsmath, amsthm, amssymb}
\usepackage{color}
\usepackage{tikz}
\usepackage{subfig}
\usepackage{graphicx}
\usepackage{float}
\usepackage{enumerate}
\usepackage{sectsty}
\usepackage{natbib}

\newtheorem{theorem} {Theorem}
\newtheorem{proposition} {Proposition}
\newtheorem{lemma} {Lemma}
\newtheorem{corollary} {Corollary}

\newtheorem*{th:distortion}{Theorem~\ref{th:distortion}}
\newtheorem*{th:asymptotic_distortion}{Theorem~\ref{th:asymptotic_distortion}}
\newtheorem*{th:consistent}{Theorem~\ref{th:consistent}}

\newcommand\argmax{\mathop{\mbox{{\rm argmax}}}}

\newcommand\E{\mathop{\mbox{{\rm E}}}}

\newcommand{\ow}{\overline{w}}
\newcommand{\ox}{\overline{x}}
\newcommand{\os}{\overline{s}}
\newcommand{\oS}{\overline{S}}

\newcommand{\hw}{\hat{w}}
\newcommand{\hx}{\hat{x}}
\newcommand{\hL}{\hat{L}}
\newcommand{\heta}{\hat{\eta}}
\newcommand{\htheta}{\hat{\theta}}
\newcommand{\hq}{\hat{q}}
\newcommand{\tx}{\tilde{x}}
\newcommand{\cx}{\check{x}}
\newcommand{\Z}{\mathcal{Z}}
\newcommand{\hd}{\hat{d}^{\RMS}}
\newcommand{\hcE}{\hat{\cal E}^{\RMS}}
\def\veryTiny{\fontsize{4pt}{4pt}\selectfont}
\newcommand\B{\mbox{\veryTiny {\rm B}}}
\newcommand\KL{\mbox{\veryTiny {\rm KL}}}
\newcommand\RMS{\mbox{\veryTiny {\rm RMS}}}
\newcommand{\hX}{\hat{X}}
\newcommand{\jointD}{\psi}
\newcommand{\typeD}{\psi_{\oS}}
\newcommand{\ratingsD}{\psi_{S}}
\newcommand{\jointDSet}{\Psi}
\newcommand{\typeDSet}{\Psi_{\oS}}
\newcommand{\jointDH}{\hat{\psi}}
\newcommand{\typeDH}{\hat{\psi}_{\oS}}
\newcommand{\mJointD}{\pi}
\newcommand{\mTypeD}{\pi_{\oS}}
\newcommand{\mRatingsD}{\pi_{S}}
\newcommand{\hJointD}{\mu}
\newcommand{\hTypeD}{\mu_{\oS}}
\newcommand{\hRatingsD}{\mu_{S}}
\newcommand{\aJointD}{\phi}
\newcommand{\aTypeD}{\phi_{\oS}}
\newcommand{\aRatingsD}{\phi_{S}}
\newcommand{\nbJointD}{\phi}

\newcommand{\nbTypeD}{\phi_{\oS}}
\newcommand{\nbTypeDSet}{\Phi}
\newcommand{\Kernel}{{\cal K}}
\newcommand{\kernel}{k}
\newcommand{\osMax}{{\os_{\mbox{\veryTiny {\rm max}}}}}
\newcommand{\osMin}{{\os_{\mbox{\veryTiny {\rm min}}}}}

\newcommand{\emailhref}[1]{\href{mailto:#1}{\tt #1}} 
\usepackage{setspace}
\begin{document}

\bibliographystyle{plain}

\title{Manipulation Robustness of Collaborative Filtering Systems}

\author{
Xiang Yan \\
Stanford University \\
\emailhref{xyan@stanford.edu} \\
\and
Benjamin Van Roy \\
Stanford University \\
\emailhref{bvr@stanford.edu} \\
}

\maketitle

\singlespacing

\begin{abstract}
A collaborative filtering system recommends to
users products that similar users like.
Collaborative filtering systems influence purchase decisions,
and hence have become targets of manipulation by unscrupulous
vendors. We provide theoretical and empirical results demonstrating
that while common nearest neighbor algorithms, which are widely used in commercial systems,
can be highly susceptible to manipulation, two classes of collaborative filtering
algorithms which we refer to as {\it linear} and {\it asymptotically linear}
are relatively robust. These results provide guidance for the design of future collaborative
filtering systems.
\end{abstract}

\doublespacing

\section{Introduction}

While the expanding universe of products available via Internet commerce
provides consumers with valuable options, sifting through the numerous alternatives to
identify desirable choices can be challenging.  Collaborative filtering (CF) systems aid
this process by recommending to users products desired by similar individuals.

At the heart of a CF system is an algorithm that predicts whether a given user will
like various products based on his past behavior and that of other users.  Nearest neighbor
(NN) algorithms, for example, have enjoyed wide use in commercial
CF systems, including those of Amazon, Netflix, and Youtube \cite{Bennett, Linden, Ryan}.
A prototypical NN algorithm stores each user's history, which may include, for instance, his product ratings and
purchase decisions.  To predict whether a particular user will like a particular product, the algorithm
identifies a number of other users with similar histories.
A prediction is then
generated based on how these so-called neighbors have responded to the product.  This prediction could
be, for example, a weighted average of past ratings supplied by neighbors.

Because purchase decisions are influenced by CF systems, they have become targets of
manipulation by unscrupulous vendors.  For instance, a vendor can create multiple online
identities and use each to rate his own product highly and competitors' products poorly.
As an example, Amazon's CF system was manipulated
so that users who viewed a spiritual guide written by a well-known Christian
evangelist were subsequently recommended a sex manual for gay men \cite{Olsen}.
Although this incident may not have been driven by commercial motives,
it highlights the vulnerability of CF systems.
The research literature
offers further empirical evidence that NN algorithms are susceptible to manipulation \cite{Burke2005a, Lam, Mehta2008, Mobasher2005, Mobasher2006, OMahony, Sandvig, Zhang}.

In order to curb manipulation, one might consider authenticating each user by asking for, say, a credit card number
to limit the number of fake identities. This may be effective in some situations.
However, in web services that do not facilitate financial transactions, such as Youtube,
requiring authentication would intrude privacy and drive users away.
One might also consider using only customer purchase data, when they are available, as a basis for recommendations
because they are likely generated by honest users. Recommendation quality may be improved, however, if higher-volume data
such as page views are also properly utilized.

In this paper, we seek to understand the extent to which manipulators can hurt the performance
of CF systems and how CF algorithms should be designed to abate their influence.
We find that, while NN algorithms can be quite sensitive to manipulation, CF algorithms that carry
out predictions based on a particular class of probabilistic models are surprisingly robust.
For reasons that we will explain in the paper, we will refer to algorithms of this kind as {\it linear CF algorithms}.

We find that as a user rates an increasing number of products, the average accuracy of
predictions made by a linear CF algorithm becomes insensitive to manipulated data.
For instance, even if half of all ratings are provided by manipulators who
try to promote half of the products, predictions
for users with long histories will barely be distorted, on average. To provide some intuition for
why our results should hold, we now offer an informal argument.
A robust CF algorithm should learn from its mistakes. In particular, differences between its predictions
and actual ratings should help improve predictions on future ratings.
A linear CF algorithm generates predictions based on a probability distribution
that is a convex combination of two distributions: one that it would learn
given only data generated by honest users
and one that it would learn given only manipulated data. As a user
whose ratings we wish to predict provides more ratings,
it becomes increasingly clear which of these two distributions
better represents his preferences.
As a result, the weight placed on manipulated data diminishes and
distortion vanishes.

The main theoretical result of this paper formalizes the above argument.  In particular,
we will define a notion of distortion induced by manipulators and establish an upper bound on distortion,
which takes a particularly simple form:
$${\rm distortion} \leq \frac{1}{n}\ln\frac{1}{1-r}.$$
Here $r$ is the fraction of data that is generated by manipulators and $n$ is the number
of products that have already been rated by a user whose future ratings we wish to predict.
The bound is very general. First, it applies to all linear CF algorithms. Second, it applies
to all manipulation strategies even if manipulators coordinate their actions and produce data
with knowledge of all data generated by honest users.
The bound demonstrates that as the number of prior ratings $n$ increases, distortion vanishes.
It also identifies the number required to limit distortion to a certain level.
This offers guidance for the design of a recommendation system: the system may, for example,
assess and inform users about the confidence of each recommendation.
The system may also require a new user to rate a set number of products before
making recommendations to him. To put this in perspective, consider the following numerical example.
Suppose a CF system that accepts binary ratings
predicts future ratings correctly $80\%$ of the time in the absence of manipulation.
If $10\%$ of all ratings are provided by manipulators, according to our bound, the system can maintain a $75\%$
rate of correct predictions by requiring each new user to rate at least $21$ products
before receiving recommendations.

To broaden the scope of our analysis, we will also study CF algorithms that
behave like linear CF algorithms asymptotically as the size of the training
set grows.  This class of algorithms, which we refer to as {\it asymptotically linear},
is more flexible in accommodating modeling assumptions that may improve
prediction accuracy.  We will establish that a relaxed version of our distortion bound for linear CF algorithms
applies to asymptotically linear CF algorithms.

We will also show that our distortion bound does not generally hold for NN algorithms.
Intuitively, this is because prediction errors do not always improve the selection of neighbors.
In particular, as a user provides more ratings,
manipulated data that contribute to inaccurate predictions of his future ratings
may remain in the set of neighbors while
data generated by honest users may be eliminated from it. As a result,
distortion of predictions may not decrease.
We will later provide an example to illustrate this.

In addition to theoretical results, this paper provides an empirical analysis using a publicly available set of
movie ratings generated by users of Netflix's recommendation system. We produce a distorted version
of this data set by injecting manipulated ratings generated using a manipulation technique studied in prior
literature.  We then compare results from application of three CF algorithms:
an NN algorithm, a linear CF algorithm called the kernel density estimation algorithm,
and an asymptotically linear CF algorithm called the naive Bayes algorithm.
Results demonstrate that while performance of the NN algorithm
is highly susceptible to manipulation, those of kernel density estimation and naive Bayes algorithms
are relatively robust. In particular, the latter two experience distortion lower than the theoretical bound
we provide, whereas the distortion for the former exceeds it by far.

One might also wonder whether manipulation robustness of a CF algorithm comes at the expense of
its prediction accuracy. As an example, consider an algorithm that fixes predictions for all ratings
to be a constant, without regard to the training data. This algorithm is uninfluenced by manipulation
but is likely to yield poor predictions, and is therefore not useful.
In our experiments, the accuracy demonstrated by the three algorithms all seems reasonable.
This suggests that accuracy of a CF algorithm may be achieved alongside robustness.

Our theoretical and empirical results together suggest that commercial recommendation systems
using NN algorithms can be made more robust by adopting approaches that we describe.
Note that we are not proposing that real-world systems should implement the specific algorithms
we present in this paper. Rather, our analysis highlights properties of CF algorithms that lead
to robustness and practitioners may benefit from taking these properties into consideration when
designing CF systems.

This paper is organized as follows. In the next section, we discuss some related work.
In Section \ref{se:model}, we formulate a simplified model that serves as a context for studying
alternative CF algorithms.  We then establish results concerning the manipulation robustness
of NN, linear, and asymptotically linear CF algorithms in Section \ref{se:cf}.
In Section \ref{se:empirical}, we present our empirical study.
We make some closing remarks in a final section.

\section{Related Work}

Early research on CF systems focused on their performance in the absence of manipulation
\cite{Breese, Drineas, Herlocker, Kleinberg, Motwani, Moon, Sarwar, Schafer}.
Almost all work on manipulation robustness has been empirical.  For example,
\cite{Burke2005a, Lam, Mehta2008, Mobasher2005, Mobasher2006, OMahony, Sandvig, Zhang}
present studies on product ratings made publicly available by Internet commerce
sites.  In each case, manipulated ratings were injected, and CF algorithms were tested on the
altered data sets.  The results point out that NN algorithms and their variants are susceptible to manipulation.
This line of work identifies an effective manipulation scheme, which is to
create multiple identities and with each identity, provide positive ratings
on products to be promoted while rating other products in a manner indistinguishable from that of honest users.
In \cite{Mehta2008, Mobasher2006, Zhang},
algorithms based on probabilistic latent semantic analysis and
principal component analysis were tested.  It turns out that these algorithms are
asymptotically linear under certain assumptions about the data, and indeed,
empirical results in these papers suggest
that they are relatively robust to manipulation.  These prior results support the conclusions of our work.

To the best of our knowledge, the only prior theoretical work on manipulation robustness of
CF algorithms is reported in \cite{OMahony}.
This work analyzed an NN algorithm that uses the majority rating among a set of neighbors as
the prediction of a user's rating in an asymptotic regime of many users,
each of whom rates all products. Manipulators rate as honest users would
except on one fixed product. A bound is established on the algorithm's prediction error on this product's
rating as a function of the percentage of ratings provided by manipulators.
In our work, we do not require users to rate all products and do not constrain manipulators
to any particular strategies. Further, we study the performance distortion on average, rather than for
a single product.  Finally, a primary contribution of our work is in establishing manipulation robustness
of linear and asymptotically linear CF algorithms, which turn out to be superior to NN algorithms in this dimension.

Several researchers have proposed alternative approaches to abating the influence of manipulators.
In \cite{Resnick}, a mechanism is proposed where users accumulate reputations while
providing ratings that are later validated by observed product quality, and a user's influence
on ratings predictions is limited by his reputation.
In this mechanism, a bound is established on the distortion induced by any finite number of manipulators.
In \cite{Massa2006,Odonovan}, researchers propose leveraging trust relationships among users to weight
recommendations and fend off manipulation. \cite{Mehta2007a, Mobasher2007a, Sandvig, Williams}
suggest detecting manipulated ratings based on their patterns and discounting their impact.
Our work complements this growing literature.  First, additional sources of information can be integrated
into the probabilistic framework that we introduce in this paper to further enhance manipulation robustness.  Second, the
analytical methods that we develop may be useful for studying the benefits
of incorporating such information.

Distortion due to manipulation may also be viewed as a loss of utility in a sequential decision problem
induced by errors in initial beliefs.  Our analysis is based on ideas similar to those that have been used to
study the latter topic, which is discussed in \cite{Gossner}.

More broadly speaking, apart from collaborative filtering, there are other ways to aggregate users'
response to products in order to provide recommendations. Research has been performed on the manipulation
robustness of these systems as well.
To get a flavor of this line of work, see \cite{Bhattacharjee, Dellarocas, Friedman, Miller}.

\section{Model}\label{se:model}

We now formulate a simplified model that will serve as a context for assessing performance of
alternative CF algorithms. We will first define the product ratings that we work with and then introduce
measures of distortion induced by manipulators.
For the convenience of the reader, we summarize our mathematical notation in a table in Appendix \ref{app:notation}.

\subsection{Ratings Vectors}
In our model, a user selects ratings from a set $\oS$.
To simplify our discussion, we let $\oS$ be a finite subset of $[0,1]$.
For example, $\oS$ could be $\{0,1\}$ with $0$ representing a negative
rating and $1$ representing a positive rating. Note that all the results in this paper
can be easily generalized to accommodate any finite set $\oS$.
There are $N$ products, and a user's type is identified
by a vector in $\oS^N$.  Each $n$th component of this vector reflects how the user would
rate the $n$th product after inspecting it.

The CF system has access to ratings provided by $M$ identities who have rated products in the past.
The data from each $m$th identity
takes the form of a ratings vector $w^m \in S^N$, where $S = \oS \cup \{?\}$.  Here, an element of $\oS$ represents a product rating
whereas a question mark indicates that a product has not been rated.  We refer to $W = (w^1, \ldots, w^M) \in S^{N\times M}$
as the {\it training data}.  This data is used by a CF algorithm to predict future ratings.

Consider a user who is distinct from identities that generated the training data and for whom we will generate recommendations.
We will refer to such a user as an {\it active user}.  We will think of a CF algorithm as providing a probability mass function (PMF)
$p_{n,x,W}$ over $\oS$ for each triplet
$(n,x,W) \in \{1,\ldots,N\} \times S^N \times  S^{N \times M}$.  The PMF $p_{n,x,W}$ represents beliefs about
how an active user who has so far provided ratings $x$ would rate product $n$ after inspecting it.
Such an algorithm can be used to guide recommendations; for example,
the CF system might recommend to the active user the product he is most likely to rate highly among those
that he has not already rated.

\subsection{Distortion Measures}\label{su:distortion_measures}

To study the influence of manipulation, we consider a situation where a fraction $r$ of the identities are created by manipulators,
while the remaining fraction $1-r$ correspond to distinct honest users.  We denote the honest ratings vectors by
$y^1,\ldots,y^{(1-r)M} \in S^N$ and the manipulated ratings vectors by $z^1,\ldots,z^{rM} \in S^N$.
Let $Y = (y^1, \ldots, y^{(1-r)M})$ and $Z = (z^1,\ldots,z^{rM})$ so that the training data is $W = (Y, Z)$.

To assess distortion of predictions made by a CF algorithm, we consider the following thought experiment.  A
hypothetical active user begins with a ratings vector $x^0$, with each $n$th component set to $x^0_n =\ ?$,
and inspects products in an order  $\nu = (\nu_1, \ldots, \nu_N) \in \sigma_N$, where $\sigma_N$
denotes the set of permutations of $\{1,\ldots,N\}$.  After inspecting product $\nu_k$, the user rates
it by sampling from the PMF $p_{\nu_k, x^{k-1}, Y}$.  An updated ratings vector $x^k$ is generated
by incorporating this new rating in $x^{k-1}$.  This stochastic process reflects how we would think honest
users behave based on the CF algorithm and uncorrupted data set $Y$.  We
introduce the following measure of distortion, which we refer to as {\it Kullback-Leibler (KL) distortion}:
$$d_n^{\KL}(p, \nu, Y, Z) = \frac{1}{n} \sum_{k=1}^n \E\left[D\left(p_{\nu_k,x^{k-1},Y} \mid\mid p_{\nu_k,x^{k-1},(Y,Z)}\right)\right],$$
where $D$ denotes Kullback-Leibler divergence with the natural log. That is, for any two PMFs
$p$ and $q$ over support $U$, $D(p \mid\mid q) = \sum_{u \in U} p(u) \ln \left( p(u) / q(u) \right)$.
This measure of the difference between PMFs is commonly used in information theory.

For each $k$, the PMF $p_{\nu_k,x^{k-1},Y}$ represents the prediction that would be made in the absence of manipulators, whereas
$p_{\nu_k,x^{k-1},(Y,Z)}$ is what it becomes as a consequence of manipulation.  Hence,
$D\left(p_{\nu_k,x^{k-1},Y} \mid\mid p_{\nu_k,x^{k-1},(Y,Z)}\right)$ measures the extent to which the manipulated data $Z$
influences the prediction.  We take the expectation of this quantity, with $x^{k-1}$ distributed as the CF algorithm
would have predicted if the data set were not corrupted by manipulated data.
KL distortion $d_n^{\KL}(p, \nu, Y,Z)$ averages these terms over the first $n$ inspected products.

Some algorithms such as NN algorithms generate predictions not in the form of PMFs, but as scalars that
may be interpreted as the means of PMFs.
For these algorithms,
it may be more suitable to measure manipulation impact in terms of {\it root-mean-squared (RMS) distortion}:
$$d_n^{\RMS}(p, \nu, Y, Z) = \sqrt{\frac{1}{n} \sum_{k=1}^n \E\left[\left(\tx_{\nu_k,x^{k-1},Y} - \tx_{\nu_k,x^{k-1},(Y,Z)}\right)^2 \right]},$$
where $\tx_{\nu_k,x^{k-1},Y}$ and $\tx_{\nu_k,x^{k-1},(Y,Z)}$ denote the scalar predictions of $\ox_{\nu_k}$
by the algorithm based on ratings history $x^{k-1}$ and data sets $Y$ and $(Y,Z)$, respectively.
Note that if the algorithm generates PMFs as predictions, $\tx_{\nu_k,x^{k-1},Y}$ and $\tx_{\nu_k,x^{k-1},(Y,Z)}$
would be expectations of $\ox_{\nu_k}$ taken with respect to
$p_{\nu_k,x^{k-1},Y}$ and $p_{\nu_k,x^{k-1},(Y,Z)}$, respectively.
The expectation in the definition of RMS distortion is taken with $x^{k-1}$ distributed as the CF
algorithm would have predicted based on $Y$.
RMS distortion may offer a more
transparent assessment than KL distortion because the former computes
how much scalar predictions change in the same unit as the predictions themselves.
RMS distortion is bounded by a function of KL distortion:
$$d_n^{\RMS}(p, \nu, Y, Z) \leq \sqrt{\frac{1}{2} d_n^{\KL}(p, \nu, Y,Z)}.$$
This is shown in Proposition \ref{pr:error_bound} in Appendix \ref{app:distortion_measures}.

To offer an intuitive interpretation for RMS distortion, we consider a setting where users provide binary
ratings and the CF system offers binary predictions based on the PMFs that it generates.
That is, we set $\oS=\{0,1\}$. Given training data $Y$, for a user with ratings history $x^{k-1}$,
the system generates a prediction of $\cx_{\nu_k,x^{k-1},Y} = 1$ for product $\nu_k$ if $p_{\nu_k,x^{k-1},Y}(1) \ge 1/2$
and generates a prediction of $\cx_{\nu_k,x^{k-1},Y} = 0$ otherwise.
Similarly, we denote $\cx_{\nu_k,x^{k-1},(Y,Z)}$ as the binary prediction based on $(Y,Z)$.
We define the following {\it binary prediction distortion}:
$$d_n^{\B}(p,\nu,Y,Z) = \frac{1}{n} \sum_{k=1}^n \left( \Pr \left(\ox_{\nu_k} = \cx_{\nu_k,x^{k-1},Y} \right) - \Pr \left(\ox_{\nu_k} = \cx_{\nu_k,x^{k-1},(Y,Z)}\right) \right),$$
where each $x_{\nu_k}$ is distributed according to $p_{\nu_k,x^{k-1},Y}$ and $x^{k-1}$ is distributed
as the CF algorithm would have predicted based on $Y$.
This quantity captures the average decrease in the probability of correct predictions, induced by
manipulation. It turns out that binary prediction distortion is bounded by RMS distortion:
\begin{equation}\label{eq:binary_bound}
d_n^{\B}(p,\nu,Y,Z) \le d_n^{\RMS}(p, \nu, Y, Z).
\end{equation}
Proved in Proposition \ref{pr:binary_bound} in Appendix \ref{app:distortion_measures}, this result offers an
interpretation of RMS distortion as an upper bound on the drop in the probability of correct predictions
in a binary setting.

One might wonder why we choose our particular distortion measures over other candidates. For instance,
one option is to consider the top $n$ most desirable products based on predictions, and define as distortion
some measure of their quality change due to the manipulated samples.
One reason why we prefer KL and RMS distortions is that they are convex functions of predictions while
this measure is not. As such, this measure is difficult to analyze. Further, as will be discussed in
Section \ref{su:results}, in a recent competition of CF algorithms,
Netflix uses RMS error to assess their prediction accuracies \cite{Netflix}. This suggests that commercial
CF algorithms are typically designed to minimize convex measures of error. Our choice of distortion
measures is in line with this approach.

Another option one might consider is to measure the worst distortion over all products.
This may be a reasonable choice if there is one manipulator interested in distorting ratings
of one product. Since we model a situation where the CF system does not know the number or
the objectives of manipulators, however, we would like to characterize the overall distortion experienced
by all products, and KL and RMS distortions capture that better than the worst distortion does.
Note that one implication of our choice is that our robustness results will
pertain to the overall distortion, rather than distortion on individual product ratings.
As such, our algorithms will not provide guarantees on whether any particular individual product's ratings
will be influenced significantly by manipulators.

\section{Collaborative Filtering Algorithms}\label{se:cf}

In this section, we first introduce the notion of {\it probabilistic CF algorithms}.
We then describe two classes of such algorithms, namely
linear and asymptotically linear CF algorithms, and analyze their robustness to manipulation.
Finally, we discuss nearest neighbor algorithms and their susceptibility to manipulation.

\subsection{Probabilistic Collaborative Filtering Algorithms}

A probabilistic CF algorithm carries out predictions based on a probabilistic model of how the training data is
generated.  We will model training data as being generated in the following way.  First,
user types $\ow^m \in \oS^N$ are sampled i.i.d. from some PMF.  Then, $w^m \in S^N$ is sampled from a
conditional PMF, conditioned on $\ow^m$, which for each $n$ assigns either $w^m_n = ?$ or
$w^m_n = \ow^m_n$.  Note that this model
allows for dependence between the type of a user and the products he chooses to rate.  This accommodates, for example,
systems in which users tend to inspect and rate only products that they care for.
Given a PMF $\jointD$ over $\oS^N \times S^N$, we denote by $\typeD$ and $\ratingsD$ the marginal
PMFs over $\oS^N$ and $S^N$, respectively.

We will call a CF algorithm $p$ {\it probabilistic} if for each $W$ there exists a PMF $\jointDH^{p,W}$ over $\overline{S}^N \times S^N$
such that for each $n$ and $x$, $p_{n,x,W}$ is the marginal PMF of $\ox_n$ conditioned on $x$,
with respect to the joint PMF $\jointDH^{p,W}$.
From here on, we will denote by $\typeDH^{p,W}$ the marginal PMF over $\oS^N$ of the joint PMF $\jointDH^{p,W}$
corresponding to a probabilistic CF algorithm $p$ and training set $W$.

\subsection{Linear Collaborative Filtering Algorithms}

We say that a probabilistic CF algorithm $p$ is {\it linear} if for any $W_1 \in S^{N \times M_1}$ and $W_2 \in S^{N \times M_2}$,
$$\typeDH^{p,(W_1,W_2)} = \frac{M_1}{M_1 + M_2} \typeDH^{p,W_1} + \frac{M_2}{M_1 + M_2} \typeDH^{p,W_2}.$$
This definition states that the PMF $\typeDH^{p,(W_1,W_2)}$ that a linear CF algorithm $p$
generates based on training data $(W_1,W_2)$ is a convex combination of two PMFs: namely,
the PMF $\typeDH^{p,W_1}$ that it generates based on $W_1$ and the PMF $\typeDH^{p,W_2}$ that it
generates based on $W_2$.

We now examine the KL distortion that manipulators can induce on a linear CF algorithm.
Consider training data $W = (Y,Z)$ consisting of ratings vectors $Y$ from honest users and $Z$ from manipulators, with the
latter making up a fraction $r$ of the training data.  The following theorem, which is the main theoretical contribution
of this paper, establishes a bound on the resulting KL distortion.

\begin{theorem}\label{th:distortion}
Fix the number of products $N$ and let $p$ be a linear CF algorithm.
Then, for all $M$, $r \in \{0,1/M,\ldots,(M-1)/M\}$, $Y \in S^{N \times (1-r) M}$, $Z \in S^{N \times r M}$, and $\nu \in \sigma_N$,
\begin{equation*}
d_n^{\KL}(p,\nu, Y, Z) \le \frac{1}{n} \ln \frac{1}{1-r}.
\end{equation*}
\end{theorem}
\noindent This result is proved in Appendix \ref{app:linear}.

Note that the bound only depends on the number of active user ratings $n$ and the fraction of
data $r$ generated by manipulators.  Hence, it represents a worst case bound over
all linear CF algorithms $p$, the number of products $N$, the quantity $M$
and values $(Y,Z)$ of the training data,
and the order $\nu$ in which the active user rates products.
This means, for example, that it applies even if manipulators coordinate with each other and select ratings with knowledge
of the specific CF algorithm $p$, the honest ratings $Y$, and the ordering $\nu$.  This also makes the bound relevant for
realistic models of how a recommendation
system might sequence products for a user; for example, each $\nu_k$ could be the product that the CF algorithm predicts as being
most desirable among remaining ones after the user has inspected products $\nu_1, \ldots, \nu_{k-1}$.

Note that KL distortion vanishes as the number $n$ of products rated by the active user increases.
To develop intuition for why this happens, we now offer an informal argument.
Observe that $\typeD^{p,(Y,Z)} = (1-r) \typeD^{p,Y} + r\typeD^{p,Z}$.
If $\typeD^{p,Y}$ is identical to $\typeD^{p,Z}$, then $\typeD^{p,(Y,Z)}$ is equal to $\typeD^{p,Y}$
and distortion will be zero.
Otherwise, if $\typeD^{p,Y}$ and $\typeD^{p,Z}$ are different,
as an active user inspects and rates products in the manner that we define, his ratings will tend
to be distinguished as sampled from $\typeD^{p,Y}$ rather than from $\typeD^{p,Z}$.
As such, the influence of $Z$ on predictions diminishes as $n$ grows.

The bound depends on $r$ through the term $\ln(1/(1-r))$.  This term captures the dependence of KL distortion
on the fraction of data produced by manipulators.  As one would expect, this term vanishes when $r$ is set to zero.

As a corollary of Theorem \ref{th:distortion} and Proposition \ref{pr:error_bound} we have the following bound on RMS
distortion.
\begin{corollary}
\label{co:error}
Fix the number of products $N$ and let $p$ be a linear CF algorithm.
Then, for all $M$, $r \in \{0,1/M,\ldots,(M-1)/M\}$, $Y \in S^{N \times (1-r) M}$, $Z \in S^{N \times r M}$, and $\nu \in \sigma_N$,
\begin{equation*}
d_n^{\RMS}(p,\nu, Y, Z) \le \sqrt{\frac{1}{2n} \ln \frac{1}{1-r}}.
\end{equation*}
\end{corollary}

Figure \ref{fig:error_bound} illustrates how this bound depends on $r$ and $n$.
The bound can offer useful guidance.  For example, it ensures that if an active user has rated $22$ products
and no more than $10\%$ of the training data is manipulated, then the RMS distortion induced by manipulators is
less than $0.05$. In a setting where users provide binary ratings and the system generates binary predictions,
according to our bound on binary prediction distortion in (\ref{eq:binary_bound}) in Section \ref{su:distortion_measures},
the average probability of correct predictions decreases by at most $0.05$.
Hence, if a binary CF system predicts ratings correctly $80\%$ of the time in the absence of manipulation,
it can maintain this probability at $75\%$ in the presence of manipulation if it requires active users
to rate $21$ products before receiving recommendations.

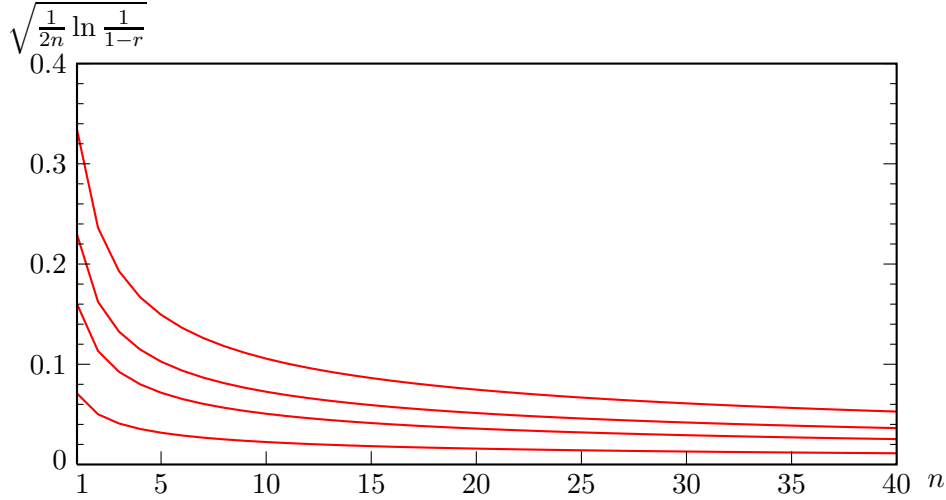
\begin{figure}[h]
  \centering
  \begin{tikzpicture}[x=0.11in,y=5.25in]
    \node[coordinate] at (1,0) (origin) {};
    \node[coordinate] at (40,0) (x) {};
    \node[coordinate] at (1,0.4) (y) {};
    \draw[red,thick] plot file {data/bnd_0.01.table};
    \draw[red,thick] plot file {data/bnd_0.05.table};
    \draw[red,thick] plot file {data/bnd_0.1.complete.table};
    \draw[red,thick] plot file {data/bnd_0.2.table};
    \draw (1,0) node[left,yshift=2pt] {$0$};
    \foreach \y in {0.1,0.2,0.3,0.4}
    \draw (1,\y) node[left] {$\y$};
    \draw (1,0.37) node[yshift=25pt] {$\sqrt{\frac{1}{2n} \ln \frac{1}{1-r}}$};
    \foreach \y in {0.1,0.2,0.3,0.4}
    {
      \draw[shift={(1,\y -| y)}] (5pt,0pt) -- (0pt,0pt);
      \draw[shift={(1,\y -| x)}] (-5pt,0pt) -- (0pt,0pt);
    }
    \foreach \y in {0.1,0.2,0.3,0.4}
    {
      \foreach \z in {0.02,0.04,0.06,0.08} {
        \draw[shift={(1,\y -| y)},shift={(0,-0.1)},shift={(0,\z)}]
        (2.5pt,0pt) -- (0pt,0pt);
        \draw[shift={(1,\y -| x)},shift={(0,-0.1)},shift={(0,\z)}]
        (-2.5pt,0pt) -- (0pt,0pt);
      }
    }
    \draw (1,0) node[below,xshift=2pt] {$1$};
    \foreach \x in {5, 10, 15, 20, 25, 30, 35, 40}
    \draw (\x,0) node[below] {$\x$};
    \draw (40,0) node[below,xshift=15pt] { $n$};
    \foreach \x in {5, 10, 15, 20, 25, 30, 35, 40}
    \draw[shift={(\x,0)}] (0pt,5pt) -- (0pt,0pt);
    \draw[thick] (origin) -- (x)
    -- (y -| x) -- (y) -- (origin);
  \end{tikzpicture}
  \caption{Bound on $d_n^{\RMS}(p,\nu,Y,Z)$ as a function of $n$. The four curves from bottom to top are for cases where $r=0.01$, $0.05$, $0.1$, and $0.2$, respectively.\label{fig:error_bound}}
\end{figure}

We will introduce examples of linear CF algorithms in Section \ref{su:kde}.

\subsection{Asymptotically Linear Collaborative Filtering Algorithms}\label{su:asymptotic_linear}

We say that a probabilistic CF algorithm $p$ is {\it asymptotically linear} if for all
PMFs $\jointD$, $\aJointD$ over $\oS^N \times S^N$, $r \in [0,1]$, and $\epsilon > 0$,
$$ \lim_{m \rightarrow \infty} \Pr \left( D\left( \left( (1-r) \typeDH^{p,U_{m}} + r \typeDH^{p,V_{m}} \right)
\mid\mid \typeDH^{p,\left( U_{m}, V_{m} \right)} \right) \ge \epsilon \right) = 0,$$
where for each $m$, $U_{m} = (u^1,\ldots,u^{m-l}) \in S^{N \times {(m-l)}}$ and
$V_{m} = (v^1,\ldots,v^l) \in S^{N \times l}$, $l \sim \mbox{Binomial}(m,r)$,
and $u^1,\ldots,u^{m-l} \sim \ratingsD$ and $v^1,\ldots,v^l \sim \aRatingsD$ are i.i.d. sequences.

To understand the preceding definition, we think of training data $(U_{m}, V_{m})$ for each $m$ as
generated in the following way: with probability $1-r$, a ratings vector is sampled from $\ratingsD$,
which we denote as $u_i$ for an appropriate $i$, and with probability $r$, a ratings vector is sampled from $\aRatingsD$,
which we denote as $v_i$ for an appropriate $i$. As $m$ grows, an asymptotically linear CF algorithm behaves like a
linear CF algorithm in that the PMF $\typeDH^{p,\left( U_{m}, V_{m} \right)}$ that it generates based on data
$(U_{m}, V_{m})$ converges in probability to a convex combination of two PMFs: namely, the PMF $\typeDH^{p,U_{m}}$
that it would generate based on the $\typeD-$sampled set $U_{m}$ and the PMF $\typeDH^{p,V_{m}}$ that it would
generate based on the $\aTypeD-$sampled set $V_{m}$.
By an application of the weak law of large numbers,
it can be shown that all linear CF algorithms are asymptotically linear.

We can also show that asymptotically linear CF algorithms are asymptotically robust, in a sense to be made
precise later. It turns out that this result applies to a broader range of
practical algorithms that are asymptotically linear in a more restricted sense, which we now define.
Consider a set $\jointDSet$ of joint PMFs over $\oS^N \times S^N$. We say that a probabilistic CF algorithm
$p$ is {\it asymptotically linear with respect to $\jointDSet$} if for all
PMFs $\jointD$, $\aJointD \in \jointDSet$, $r \in [0,1]$, and $\epsilon > 0$,
$$ \lim_{m \rightarrow \infty} \Pr \left( D\left( \left( (1-r) \typeDH^{p,U_{m}} + r \typeDH^{p,V_{m}} \right)
\mid\mid \typeDH^{p,\left( U_{m}, V_{m} \right)} \right) \ge \epsilon \right) = 0,$$
where for each $m$, $U_{m} = (u^1,\ldots,u^{m-l}) \in S^{N \times {(m-l)}}$ and
$V_{m} = (v^1,\ldots,v^l) \in S^{N \times l}$, $l \sim \mbox{Binomial}(m,r)$,
and $u^1,\ldots,u^{m-l} \sim \ratingsD$ and $v^1,\ldots,v^l \sim \aRatingsD$ are i.i.d. sequences.

The following theorem and corollary characterize the robustness of asymptotically linear CF algorithms.

\begin{theorem}\label{th:asymptotic_distortion}
Fix the number of products $N$ and a set $\jointDSet$ of joint PMFs over $\oS^N \times S^N$.
Let $p$ be a CF algorithm asymptotically linear with respect to $\jointDSet$.
Then, for all $\hJointD^*, \mJointD^* \in \jointDSet$, $r \in [0,1)$, $\nu \in \sigma_N$,
$n \in \Z_+$, and $\epsilon > 0$,
\begin{equation*}
\lim_{m \rightarrow \infty} \Pr \left( d_n^{\KL}(p,\nu, Y_{m}, Z_{m}) \ge \frac{1}{n} \ln \frac{1}{1-r} + \epsilon \right) = 0,
\end{equation*}
where, for each $m$, $Y_{m} = (y^1,\ldots,y^{m-l}) \in S^{N \times {(m-l)}}$,
$Z_{m} = (z^1,\ldots,z^l) \in S^{N \times l}$, $l \sim \mbox{Binomial}(m,r)$,
and $y^1,\ldots,y^{m-l} \sim \hRatingsD^*$ and $z^1,\ldots,z^l \sim \mRatingsD^*$ are i.i.d. sequences.
\end{theorem}

\begin{corollary}\label{co:asymptotic_distortion}
Fix the number of products $N$ and a set $\jointDSet$ of joint PMFs over $\oS^N \times S^N$.
Let $p$ be a CF algorithm asymptotically linear with respect to $\jointDSet$.
Then, for all $\hJointD^*, \mJointD^* \in \jointDSet$, $r \in [0,1)$, $\nu \in \sigma_N$,
$n \in \Z_+$, and $\epsilon > 0$,
\begin{equation*}
\lim_{m \rightarrow \infty} \Pr \left( d_n^{\RMS}(p,\nu, Y_{m}, Z_{m}) \ge \sqrt{\frac{1}{2n} \ln \frac{1}{1-r}} + \epsilon \right) = 0,
\end{equation*}
where, for each $m$, $Y_{m} = (y^1,\ldots,y^{m-l}) \in S^{N \times {(m-l)}}$,
$Z_{m} = (z^1,\ldots,z^l) \in S^{N \times l}$, $l \sim \mbox{Binomial}(m,r)$,
and $y^1,\ldots,y^{m-l} \sim \hRatingsD^*$ and $z^1,\ldots,z^l \sim \mRatingsD^*$ are i.i.d. sequences.
\end{corollary}

\noindent Theorem \ref{th:asymptotic_distortion} is proved in Appendix \ref{app:asymptotically_linear} and
Corollary \ref{co:asymptotic_distortion} follows from Theorem \ref{co:asymptotic_distortion}
and Proposition \ref{pr:error_bound}.
These results state that for any CF algorithm $p$ asymptotically linear with respect to $\jointDSet$
and any fixed PMFs $\hRatingsD^*$, $\mRatingsD^* \in \jointDSet$,
as honest users and manipulators sample more data from them, with high probability,
the distortion bounds for linear CF algorithms in Theorem \ref{th:distortion} and Corollary \ref{co:error}
will also apply to $p$ and in particular, distortion will vanish as $n$ grows.

The intuition behind these results is similar to that for Theorem \ref{th:distortion}. In particular,
given sufficient data, the learned PMF $\typeDH^{p,(Y_m,Z_m)}$ should closely approximate $(1-r) \hTypeD^* + r \mTypeD^*$.
If $\hTypeD^*$ and $\mTypeD^*$ are similar, then $\typeDH^{p,(Y_m,Z_m)}$
should be close to $\hTypeD^*$ and distortion should be close to zero.
On the other hand, if $\hTypeD^*$ are $\mTypeD^*$ are significantly different,
as an active user provides more ratings, it will be increasingly clear that they are sampled from $\hTypeD^*$
rather than $\mTypeD^*$, and distortion will diminish.

We now study a class of asymptotically linear CF algorithms, which
converge to the true PMF of user types under certain assumptions about the training data.
For starters, we say a set $\jointDSet$ of PMFs over $\oS^N \times S^N$ is {\it identifiable}
if all distinct PMFs $\jointD, \aJointD \in \jointDSet$
have distinct ratings marginals $\ratingsD$ and $\aRatingsD$.
Given an identifiable set $\jointDSet$,
we say that a probabilistic CF algorithm $p$ is {\it consistent with respect to $\jointDSet$}
if for all $\jointD \in \jointDSet$ and $\epsilon > 0$,
$$\lim_{m \rightarrow \infty}
\Pr \left( D\left( \typeDH^{p,W_m} \mid\mid \typeD \right) \ge \epsilon \right) = 0,$$
where for each $m$, $W_m=(w_1,\ldots,w_m) \in S^{N \times m}$ is generated independently
and in particular, $w_1,\ldots,w_m \sim \ratingsD$ is an i.i.d. sequence. This definition is
meant to capture algorithms that converge to $\ratingsD$ and recover its unique corresponding joint PMF
$\jointD \in \jointDSet$ and type marginal $\typeD$ as the $\ratingsD-$sampled training data grows.
The following theorem states the setting in which a consistent CF algorithm is asymptotically linear.
\begin{theorem}\label{th:consistent}
Any probabilistic CF algorithm consistent with respect to an identifiable and convex set
$\jointDSet$ is asymptotically linear with respect to $\jointDSet$.
\end{theorem}
\noindent The preceding result, proved in Appendix 
\ref{app:asymptotically_linear}, together with the definition of consistent algorithms
and Theorem \ref{th:asymptotic_distortion}
imply that if data are sampled i.i.d. from some PMF $\jointD$ in an
identifiable and convex set $\jointDSet$, then a consistent algorithm with respect to $\jointDSet$
would provide guarantees on both prediction accuracy and robustness to manipulation as training data grows.
In practice, even if it is unclear whether the identifiability and convexity conditions hold, as a starting point,
one might still apply a consistent CF algorithm, with the hope that it will deliver
reasonable accuracy and robustness.
In Section \ref{se:empirical},
we will empirically evaluate a consistent CF algorithm called the naive Bayes algorithm.

\subsection{Nearest Neighbor Algorithms}\label{su:NN}

Nearest neighbor algorithms, widely used in commercial CF systems \cite{Bennett, Linden, Ryan}, generally come in two classes.
The first class predicts a user's ratings based on those provided by similar users, referred to as neighbors.
The second class makes predictions on a product based on ratings that the user has provided on similar products,
which can also be viewed as neighbors.
In this section, we study a simple NN algorithm of the first class and the extent to which its
predictions can be distorted by manipulators.  We show that the bounds of the previous section
do not apply to this NN algorithm, and unlike the case of linear CF algorithms,
distortion does not generally diminish as the active user inspects and rates products.
Though our analysis focuses on a particular NN algorithm,
the resulting insights apply more broadly and in particular, to NN algorithms of the second class as well.

We study the case of binary ratings.  NN algorithms identify and weight
neighbors using a similarity measure.  We will consider a similarity measure that increases by one for each
pair of consistent ratings and decreases by one for each pair of inconsistent ratings:
$$s(x,y) = |\{1 \le n \le N: x_n = y_n \neq\ ?\}| - |\{1 \le n \le N :\ ? \neq x_n \neq y_n \neq\ ?\}|,$$
for any pair of ratings vectors $x,y \in S^N$.

We consider an NN algorithm that predicts the future rating of product $n$ for
a user with ratings vector $x$ by carrying out the following steps.  First, the algorithm
identifies the subset of the training data samples that offer ratings for product $n$.
If this subset is empty, the NN algorithm optimistically predicts a rating of $1$.
Otherwise, from among these ratings vectors, the ones most similar to $x$ are identified.
We denote the resulting set of neighbors, which should be a singleton unless there is a tie,
by ${\cal N}(n,x,W)$.  Finally, an average of their ratings for product $n$
forms the prediction:
$$\tx_{n,x,W} = \frac{\sum_{w \in {\cal N}(n,x,W)} w_n}{|{\cal N}(n,x,W)|}.$$
Our observations extend to other more complicated similarity metrics
and neighbors selection methods.  However, we focus on this particular case in order to keep our analysis clean.

We now consider a simple setting that facilitates analysis of RMS distortion in our NN algorithm.
We are interested in how RMS distortion changes as the number of ratings $n$ provided by an active user
grows.  Since $n$ cannot exceed the number of products $N$, we will define an ensemble of
models indexed by $N$.
To facilitate our construction, we will only consider even $N$.

To keep things simple, we restrict attention to a situation where honest users agree
on the ratings of all products.  In particular, there is a single user type $\ox^{\rm odd}$
which rates odd-indexed products $1$ and even-indexed products $0$.  The user type
PMF $\hTypeD^*$ assigns all probability to this vector.  Each honest ratings vector $y^m$
is generated by sampling a random set of odd numbers between $1$ and $N-1$,
then for each sample $k$, replacing components $k$ and $k+1$ of $\ox^{\rm odd}$
with question marks.  We assume that the honest ratings $Y$ of training data is such that
each set of odd numbers between $1$ and $N-1$ is sampled exactly once.
That is, each element of $Y$ corresponds to an element of the set $\{(1,0),(?,?)\}^{N/2}$.
As such, there are $2^{N/2}$ honest ratings vectors.

Recalling the setting that we use for assessing distortion,
we now consider an active user who inspects products in the ordering $\nu = (1, \ldots,N)$,
rating each based on the prediction of the NN algorithm.
It is easy to see that when there are no manipulators, the NN algorithm perfectly predicts
$\tx_{k,x^{k-1},Y} = \ox^{\rm odd}_k$, and therefore, after the user inspects $k$
products, his ratings history $x^k$ has
$x^k_j = \ox^{\rm odd}_j$ for $j \leq k$ and $x^k_j =\ ?$ for $j > k$.

We assume that manipulators produce one half of the training data.  For each honest ratings vector
$y^m$, manipulators produce a ratings vector $z^m$ which agrees with $y^m$ on all products
rated by $y^m$.  However, question marks in $y^m$ are replaced by $1$ for even indices and
$0$ for odd indices. That is, each $z_m$ corresponds to an element of the set $\{(1,0),(0,1)\}^{N/2}$.

Suppose $k$ is even.
Given $x^k$, the NN algorithm predicts what the active user's rating will be for
product $k+1$.  To do this, it identifies neighbors ${\cal N}(k+1,x^k,(Y,Z))$, which
includes the following subsets of the training data:
\begin{itemize}
\item A set $Y_1$ which consists of
honest ratings vectors $y^m$ where $y^m_j \neq\ ?$ for $j \leq k+1$.
\item A set $Z_1$ that, for each $y^m \in Y_1$, includes the corresponding manipulated
vector $z^m \in Z$.
\item A set $Z_2$ which consists of each manipulated ratings vector $z^m$ such that
the corresponding honest ratings vector $y^m$ has $y^m_j \neq\ ?$ for $j \leq k$
and $y^m_{k+1} =\ ?$.
\end{itemize}
Note that each of these sets is of cardinality $2^{(N - k)/2 - 1}$.
Vectors in $Y_1$ and $Z_1$ correctly rate product $k+1$ as $1$,
whereas vectors in $Z_2$ incorrectly rate it as $0$.  As a consequence, the prediction for product $k+1$ is
$\tx_{k+1,x^k,(Y,Z)} = 2/3$ and the resulting squared error is
$$\left(\ox^{\rm odd}_{k+1} - \tx_{k+1,x,(Y,Z)}\right)^2 = \frac{1}{9}.$$

The preceding argument applies for all even $k$.  For odd $k$, it is easy to show that the NN algorithm correctly
predicts $\ox^{\rm odd}_{k+1} = 0$.  It follows that the RMS distortion for even $n$ is
$$d_n^{\RMS}(p,\nu, Y, Z) = \sqrt{\frac{1}{n} \sum_{k=1}^n \E\left[\left(
\ox^{\rm odd }_k - \tx_{k,x^{k-1},(Y,Z)}\right)^2 \right]}
= \frac{1}{3 \sqrt{2}}.$$

The preceding example shows that the RMS distortion of an NN algorithm for $r = 1/2$ does not decrease
as $n$ grows.
This happens because manipulated data are strategically generated to be sufficiently
similar to honest data so that no matter how many ratings an active user provides,
manipulated ratings vectors will make up a fixed fraction of the neighbors
and consequently induce a significant amount of distortion.

In contrast, Corollary \ref{co:error} establishes that linear
CF algorithms exhibit a more graceful behavior, with RMS distortion vanishing as $n$ increases.
This is not to say it is impossible to design an NN algorithm that exhibits a more desirable behavior
when applied to our example.  However,
it is difficult to know for sure whether a given variation
will behave gracefully in all relevant situations.

\subsection{Discussion}
We now provide an intuitive explanation for why linear CF algorithms
should be robust to manipulation relative to NN algorithms.
First note that robustness depends on how a CF algorithm learns from its mistakes.
In particular, a robust algorithm should notice as it observes differences
between its predictions and an active user's ratings that certain things
learned from the data set are hurting rather than improving its predictions.

Recall that a linear CF algorithm $p$ generates based on the training set $(Y,Z)$
a PMF $\typeDH^{p,(Y,Z)}$ that is a convex
combination of $\typeDH^{p,Y}$ and $\typeDH^{p,Z}$, which are PMFs that the algorithm would generate
based on $Y$ and $Z$, respectively.
As an active user rates more products,
it will be increasingly clear by probabilistic inference
that his ratings $x$ are sampled from $\typeDH^{p,Y}$. In effect, inaccurate predictions induced
by $Z$ will increase the weight on $\typeDH^{p,Y}$ in the conditional PMF of ratings $\typeDH^{p,(Y,Z)}(\cdot|x)$
conditioned on observed ratings $x$. And this makes future predictions more accurate.

In an NN algorithm, on the other hand, inaccurate predictions do not
generally improve further predictions.  In particular,
manipulated ratings vectors that contribute to inaccuracies
may remain in the set of neighbors while honest ratings vectors may be eliminated from it.
In the example in Section \ref{su:NN}, for instance,
manipulated data are generated so that no matter how long an active user's ratings
history is, each honest ratings vector selected as a neighbor
has a manipulated counterpart that is as similar, and hence also selected as a neighbor.
Consequently, as an active user provides more
ratings, the numbers of honest and manipulated neighbors both decrease and stay equal.
As a result, inaccurate predictions do not decrease future distortion.

\section{Empirical Study}\label{se:empirical}

In this section, we present our empirical findings on
the manipulation robustness of NN, linear, and asymptotically linear CF algorithms.
We first introduce the data set that we worked with and then describe the methods
we used to evaluate robustness.

\subsection{Data Set}
We obtained a set of movie ratings provided by users,
made publicly available by Netflix's recommendation system.
Each rating is an integer between $1$ and $5$,
which we normalized to be in $\{0,0.25,0.5,0.75,1\}$ so that the analysis and results in our paper
apply directly.
We randomly sampled from the data set
$5000$ users, who have provided $200000$ ratings of $500$ movies.
We then randomly chose $4000$ of these users and for the purpose of our experiments,
treated them as honest users and their ratings as a training set $Y$.
We used the ratings of the other $1000$ users as a test set, which we refer to as $X$.
We then generated three separate sets of $444$, $1714$, and $4000$ manipulated ratings vectors, respectively.
Each set, which we refer to as $Z$ for simplicity of discussion, is generated
to promote $50\%$ of the movies by using a technique
reported to be effective in the literature \cite{Burke2005a, Lam, Mehta2008, Mobasher2005, Mobasher2006, OMahony, Sandvig, Zhang}.
Specifically, we randomly sampled $250$ of the $500$ movies in $Y$ and let each manipulated ratings vector in each $Z$
assign the highest ratings to these movies, and assign a random rating to each of the other movies, sampled from the movie's
empirical marginal PMF of ratings in $Y$. We then replaced a random subset of ratings in $Z$ with question marks so that
its fraction of question marks matches that in $Y$.
Manipulated ratings vectors generated this way are meant to be similar to honest ratings
vectors except on movies to promote.

\subsection{Evaluation Methods}\label{su:evaluation}
To test the robustness of each CF algorithm $p$, we treated ratings in $X$ as ratings that an active user
would provide and let $p$
predict them. Specifically, we fixed $n$ and for each ratings vector
$x \in X$, identified $n$ random products that it has assigned ratings to and randomly permuted them to form an ordering
$\nu^x=(\nu_1^x,\ldots,\nu_n^x)$. For each $k \le n$,
let $x^{k-1}$ be a ratings vector that agrees with $x$ on products $\nu_1^x,\ldots,\nu_{k-1}^x$ and
assigns question marks to the other products. An algorithm $p$ is then used to generate a scalar prediction
$\tx_{\nu_k^x,x^{k-1},Y}$ for the rating of product $\nu_k^x$ based on $x^{k-1}$ and the honest data set $Y$.
Similarly, a prediction based on a training set $(Y,Z)$
corrupted by manipulated ratings is denoted by $\tx_{\nu_k^x,x^{k-1},(Y,Z)}$.
To assess influence due to manipulation, for each $n$, we computed the following quantity, which we will refer to as
{\it empirical RMS distortion}:
$$\hd_n(p,\nu^X,X,Y,Z) = \sqrt{ \frac{1}{|X|} \sum_{x \in X} \frac{1}{n} \sum_{k=1}^n \left(\tx_{\nu_k^x,x^{k-1},Y} - \tx_{\nu_k^x,x^{k-1},(Y,Z)}\right)^2}.$$
Here, $\nu^X = \{(\nu_1^x,\ldots,\nu_n^x) : x \in X\}$.
The empirical RMS distortion measures changes of predictions for products
rated by active users. It is similar to the RMS
distortion $d_n^{\RMS}(p,\nu,Y,Z)$ that we defined earlier, with one difference:
whereas $d_n^{\RMS}(p,\nu,Y,Z)$ samples each $\ox_{\nu_k}$
from the PMF $p_{\nu_k,x^{k-1},Y}$ that the algorithm generates based on $Y$, $\hd_n(p,\nu^X,X,Y,Z)$
uses elements of $X$ as samples.
We used empirical RMS distortion rather than RMS distortion to assess algorithms in
our empirical study because computing RMS distortion would take too long, requiring a
running time exponential in the number of products $n$ rated by an active user.
Further, if a CF algorithm generates a nearly correct distribution in the
absence of manipulation, its empirical RMS distortion will be close to its RMS distortion.

One might also wonder whether high robustness of CF algorithms stems from high prediction accuracy or
comes at the expense of it.
To better understand the relationship between these two performance measures,
we also computed the following RMS error for each CF algorithm, which we will refer to as
{\it empirical RMS prediction error}:
$$\hcE_n(p,\nu^X,X,Y) = \sqrt{\frac{1}{|X|} \sum_{x \in X} \frac{1}{n} \sum_{k=1}^n \left(x_{\nu_k^x} - \tx_{\nu_k^x,x^{k-1},Y} \right)^2}.$$
This quantity computes the RMS error of predictions for ratings in $X$ when the
algorithm uses $Y$ as training data.

For algorithms that we tested, we tuned some of their parameters by cross validation.
This is a technique that selects parameter values based on the performance of the corresponding algorithm
on out-of-sample data
in order to estimate their performance on future data.
Specifically, we randomly sampled $20\%$ of the users in $Y$. We treated their ratings
as a validation set $V$ and generated predictions based on the remaining ratings $Y \char92 V$.
Consider a parameter $\gamma$ that we tuned for an algorithm.
For each value $\gamma'$ in a range $\Gamma$, we set $\gamma=\gamma'$,
used the corresponding algorithm to predict ratings in $V$
based on ratings in $Y \char92 V$, and computed the empirical RMS prediction
error $\hcE_n(p,\nu^V,V,Y \char92 V)$. Finally, we selected a parameter value $\gamma^*$ that
results in a minimal error.
Similarly, when using $(Y,Z)$ as the training set, we sampled the validation set $V$
from $(Y,Z)$ and for each $\gamma' \in \Gamma$,
computed the empirical RMS prediction error $\hcE_n(p,\nu^V,V,(Y,Z) \char92 V)$ and selected a $\gamma^*$.
Note that we chose to optimize for prediction accuracy rather than robustness in cross validation
because we wanted the algorithms to maintain reasonable accuracy and wanted to avoid tuning
them to be robust for specific manipulation techniques.

Overall, for each algorithm $p$, we generated multiple samples of $X$, $Y$, $Z$, and $\nu^X$, and averaged their resultant
$\hd_n(p,\nu^X,X,Y,Z)$ and $\hcE_n(p,\nu^X,X,Y)$ across samples
to obtain reliable estimates. To summarize with our notation,
$\oS=\{0,0.25,0.5,0.75,1\}$, $N=500$, $(M,r) \in \{(4444,0.1), (5714,0.3), (8000,0.5)\}$, and $1 \le n \le 40$.
We tested three CF algorithms: a linear CF algorithm called kernel density estimation,
an asymptotically linear CF algorithm called naive Bayes, and an NN algorithm called $k$ nearest neighbor.
We now present them in detail.

\subsection{Kernel Density Estimation Algorithms}\label{su:kde}

Kernel density estimation (KDE) algorithms smooth the training data and use their resultant
distribution to predict future ratings.
For an in-depth treatment of KDE algorithms, see \cite{Hastie}.
In our context, we say that a probabilistic CF algorithm $p$ is a KDE algorithm
with kernels $\{\Kernel_w: w \in S^N\}$ if for any $W \in S^{N \times M}$,
$$\typeDH^{p,W} = \frac{1}{M} \sum_{w \in W} \Kernel_w,$$
where each $\Kernel_w$ is a PMF over $\oS^N$ parameterized by a ratings vector $w$.
It turns out that any KDE algorithm is a linear CF algorithm and any linear CF algorithm is a KDE algorithm.
We will establish this in Proposition \ref{pr:KDE_linear} in Appendix \ref{app:kde_nb}.

In our experiments, we considered a KDE algorithm with kernels $\{\Kernel_w\}$ such that for each type $\ox \in \oS^N$,
$$\Kernel_{w}(\ox) = \prod_{n=1}^N \kernel_{w_n}(\ox_n),$$
where for each $s \in S$, $\kernel_s$ is a PMF over $\oS$ defined as follows.
For $s \not= ?$, $\kernel_s$ is the unique PMF that satisfies
$\kernel_s(\os) / \kernel_s(s) = \exp(- |\os - s| / \beta)$
for all $\os \in \oS$. For $s = ?$,
$\kernel_s(\os) = 1 / |\oS|$ for all $\os \in \oS$.
That is, $\kernel_s$ assigns the highest probability
to $s$ and exponentially lower probabilities to values different from $s$ if $s \not= ?$,
and assigns uniform probability to all values if $s = ?$.
It is easy to see that each $\Kernel_w$ thus defined is a PMF, and
it assigns high probability to types similar to $w$ and low probability to others.
The constant $\beta>0$ tunes the shape of $\kernel_s$, which we set to be $0.15$ in our experiments.

To predict the rating of product $\nu_n$ for a user with past ratings $x^{n-1}$,
our KDE algorithm generates a PMF $p_{\nu_n,x^{n-1},W}$, which is the conditional PMF of
$\ox_{\nu_n}$ conditioned on $x^{n-1}$ with respect to the joint PMF $\typeDH^{p,W}$, given by
$$p_{\nu_n, x^{n-1}, W}(\os) = \typeDH^{p,W} \left(\ox_{\nu_n}=\os | x^{n-1} \right)
= \frac{\sum_{w \in W} \prod_{k=1}^{n-1}\kernel_{w_{\nu_k}}(x_{\nu_k}^{n-1}) \kernel_{w_{\nu_n}}(\os) }
{\sum_{w \in W} \prod_{k=1}^{n-1}\kernel_{w_{\nu_k}}(x_{\nu_k}^{n-1}) },$$
for each $\os \in \oS$. The corresponding scalar prediction is the expectation
taken with respect to $p_{\nu_n,x^{n-1},W}$:
$$\tx_{\nu_n,x^{n-1},W} = \sum_{\os \in \oS} \os \, p_{\nu_n, x^{n-1}, W}(\os).$$

\subsection{Naive Bayes Algorithm}\label{su:nb}
A naive Bayes (NB) algorithm assumes that the true distribution of data
is a convex combination of distinct distributions in each of which features of the data are conditionally independent.
It aims to learn from training data
the weights of the combination and feature marginals within each distribution.
For a formal analysis of the algorithm and its applications to other problem settings, see \cite{Cheeseman, Domingos, John}.
We now describe a particular version of the algorithm that we used
and discuss the context in which it is consistent and asymptotically linear.

Our NB algorithm assumes that data are sampled from a joint PMF $\nbJointD$ over $\oS^N \times S^N$
such that for each $(\ow, w) \in \oS^N \times S^N$ with $\ow \rightarrow w$,
\begin{equation}\label{eq:nb}
\nbJointD(\ow,w) = \left( q^{\| w \|_?} (1-q)^{N - \| w \|_?} \right)
\left( \sum_{l=1}^L \eta_l \prod_{n=1}^N \theta_{l,n}^{(\ow_n)} \right),
\end{equation}
where $q \in [0,1)$, $L \in \Z_+$, $\eta \in T_L$, and $\theta_{l,n} \in T_{|\oS|}$ for all $l,n$.
Here, we write $\ow \rightarrow w$ for $(\ow, w)$ if for each $n$, either $w_n=\ow_n$ or $w_n=?$.
We let $\| w \|_?$ denote $|\{n:w_{n}=?\}|$.
For any $k$, we define simplex
$T_k = \{(t_1,\ldots,t_k): t_j \ge 0, \forall \, 1 \le j \le k. \sum_{j=1}^k t_j = 1 \}$.
We also let $\theta=\{\theta_{l,n}, 1 \le l \le L, 1 \le n \le N\}$.

To understand $\nbJointD$, let us consider the following generative process of ratings vectors.
Let $\nbTypeDSet = \{\nbTypeD^1,\ldots,\nbTypeD^L\}$ be $L$ PMFs over $\oS^N$ where each $\nbTypeD^l$ satisfies
$$\nbTypeD^l(\ow) = \prod_{n=1}^N \theta_{l,n}^{(\ow_n)}$$
for all $\ow \in \oS^N$. That is, each $\ow_n$ is independently distributed and is equal
to $\os \in \oS$ with probability $\theta_{l,n}^{(\os)}$. A type $\ow$ is generated by first selecting
a PMF from $\nbTypeDSet$ where each $\nbTypeD^l$ is chosen with probability $\eta_l$ and then sampling from
that PMF. A ratings vector $w$ is then generated by randomly replacing each rating $\ow_n$ by
a question mark with probability $q$, independent of the value $\ow_n$ and whether other ratings
are replaced by question marks.

The algorithm also assumes a geometric prior for $L$ and Dirichlet priors for $\eta$, $\theta$, and $q$.
Hence, the posterior probability density function (PDF) of $(L,\eta,\theta,q)$ conditioned on training data $W$ is given by
\begin{equation}\label{eq:posterior_prob}
f(L,\eta,\theta,q|W) = c f(L, \eta, \theta,q) \Pr(W | L,\eta,\theta,q),
\end{equation}
where $c$ is a normalizing constant and prior PDF
$$f(L,\eta,\theta,q) = p_{L}^{\tau}(L) f_{\eta}^L(\eta) \prod_{l,n}f_{\theta}(\theta_{l,n}) f_{q}(q), \mbox{ with}$$
$$p_{L}^{\tau}(L) = c_L e^{-\tau L},$$
$$f_{\eta}^L(\eta) = c_{\eta}^L \prod_{l=1}^L \eta_l,$$
$$f_{\theta}(\theta_{l,n}) = c_{\theta} \prod_{\os \in \oS} \theta_{l,n}^{(\os)}, \, \forall\, l,n, \mbox{ and}$$
$$f_q(q) = c_q q (1 - q),$$
and data likelihood
$$\Pr(W | L,\eta,\theta,q) =
\prod_{w \in W} \left( \left( q^{\| w \|_?} (1-q)^{N - \| w \|_?} \right)
\left( \sum_{l=1}^L \eta_l \prod_{n:w_n\not=?} \theta_{l,n}^{(w_n)} \right) \right).$$
Here, subscripts $L$, $\eta$, $\theta$, and $q$ of the functions $p_L^{\tau}$, $f_{\eta}^{L}$, $f_{\theta}$, and $f_q$
denote the parameters that the distributions are over. The superscript $\tau$ in $p_L^{\tau}$ denotes dependence on
parameter $\tau$, which controls the shape of the geometric PMF.
The superscript $L$ in $f_{\eta}^L$ denotes dependence on $L$.
$c_L$, $c_{\eta}^L$, $c_{\theta}$, and $c_q$ are normalizing constants.

The algorithm maximizes the posterior PDF over parameters
by using the expectation-maximization algorithm \cite{Dempster} and obtains
$$(\hL, \heta, \htheta, \hq) \in \argmax_{(L, \eta, \theta, q)} f(L, \eta, \theta, q | W).$$
We denote by $\typeDH^{p,W}$ the PMF over $\oS^N$ implied by $(\hL, \heta, \htheta)$, which
the algorithm uses for predictions.

In particular, a prediction for the rating of product $\nu_n$ for a user with past ratings $x^{n-1}$ is given by the PMF
$$p_{\nu_n, x^{n-1}, W}(\os) = \typeDH^{p,W} \left(\ox_{\nu_n}=\os | x^{n-1} \right)
= \frac{\sum_{l=1}^{\hL} \heta_l \prod_{k=1}^{n-1} {\htheta}_{l,\nu_k}^{(x_{\nu_k}^{n-1})} {\htheta}_{l,\nu_n}^{(\os)}}
{\sum_{l=1}^{\hL} {\heta}_l \prod_{k=1}^{n-1} {\htheta}_{l,\nu_k}^{(x_{\nu_k}^{n-1})}},$$
for each $\os \in \oS$. The corresponding scalar prediction is
$$\tx_{\nu_n,x^{n-1},W} = \sum_{\os \in \oS} \os \, p_{\nu_n, x^{n-1}, W}(\os).$$

In our experiments, we tuned bandwidth $\tau$ by cross validation over the range
$\Gamma=\{1,10,$
$100,1000,10000,100000\}$ and settled at $\tau=10000$.

We now discuss the context in which the NB algorithm is consistent and asymptotically linear.
For any $a \in [0,1)$, We let $\jointDSet^{a}$ be the set of all joint PMFs over $\oS^N \times S^N$
of the form in (\ref{eq:nb}) where $q$ is fixed to be $a$.
We establish in Proposition \ref{pr:nb_consistent} in Appendix \ref{app:kde_nb}
that for any $a$, $\jointDSet^a$ is identifiable and convex,
and the NB algorithm is consistent with respect to it.
Then, by Theorems \ref{th:asymptotic_distortion} and \ref{th:consistent},
the algorithm is asymptotically linear with respect to $\jointDSet^a$ and our distortion bounds apply.
Let $\typeDSet^a = \{\typeD: \jointD \in \jointDSet^a\}$
be the set of marginals over $\oS^N$ of PMFs in $\jointDSet^a$.
We note that for any $a$, $\typeDSet^a$ contains all PMFs over $\oS^N$.
This implies that for any PMFs $\hTypeD^*$ and $\mTypeD^*$ over $\oS^N$,
if honest and manipulated data are generated by first sampling types from these PMFs and then
independently replacing each rating by a question mark with the same probability $a$,
then the NB algorithm will be asymptotically robust as the sample size grows.
In our experiments, although the condition regarding replacement by question marks may not hold,
we still apply the NB algorithm with the hope that it will deliver reasonable robustness.

\subsection{$k$ Nearest Neighbor Algorithm}
A class of NN algorithms called $k$ nearest neighbor ($k$NN) algorithms is
frequently used as a performance benchmark in prior work \cite{Burke2005a, Lam, Zhang}. The version that we tested
works as follows.

To predict the rating of product $\nu_n$ by a user with past ratings $x^{n-1}$ where $n \ge 3$,
the algorithm identifies a set of neighbors ${\cal N}(\nu_n,x^{n-1},W)$ to be $k$ ratings vectors $w \in W$ such that
$w_{\nu_n} \not= ?$ and score highest with $x^{n-1}$ on the following similarity measure:
$$s(w, x^{n-1}) = \frac{\sum_{1\le i\le n-1: w_{\nu_i} \not=?} (w_{\nu_i} - \hw)(x_{\nu_i}^{n-1} - \hx^{n-1})}
{\sqrt{\sum_{1 \le i \le N: w_i \not=?} (w_i - \hw)^2} \sqrt{\sum_{1 \le i \le n-1} (x_{\nu_i}^{n-1} - \hx^{n-1})^2}},$$
where average ratings are given by
$$\hw = \frac{\sum_{1 \le i \le N: w_i \not=?} w_i}{|\{1 \le i \le N: w_i \not=?\}|},$$
$$\hx^{n-1} = \frac{\sum_{1 \le i \le n-1} x_{\nu_i}^{n-1}}{n-1}.$$
Note that $s$ here resembles the notion of a sample correlation coefficient.
Its numerator is the covariance between non-question-mark components of $w$ and $x^{n-1}$.
The denominator is the product of the standard deviation of non-question-mark components of $w$ and the same quantity
for $x^{n-1}$.
The algorithm then generates the following scalar prediction:
$$\tx_{\nu_n,x^{n-1},W} = \min \left\{ \osMax, \max \left\{ \osMin,
\hx^{n-1} + \frac{\sum_{w \in {\cal N}(\nu_n,x^{n-1},W)} s(w, x^{n-1}) (w_{\nu_n} - \hw)}{\sum_{w \in {\cal N}(\nu_n,x^{n-1},W)} |s(w, x^{n-1})|} \right\} \right\},$$
where $\osMax = \max \{\os: \os \in \oS\}$ and $\osMin=\min\{\os: \os \in \oS\}$.
To arrive at this quantity, for each neighbor, the difference between its rating $w_{\nu_n}$
for product $\nu_n$ and its average rating $\hw$ is first computed.
A weighted sum of these differences is then computed, where the weights are normalized similarity
measures. The user's historical ratings average $\hx^{n-1}$ is then added to the sum. The total is used as the prediction,
unless it falls outside $[\osMin,\osMax]$, in which case either $\osMin$ or $\osMax$ is used, whichever is closer.

For a user with ratings history $x^{n-1}$ where $n \le 2$, $s(\cdot,x^{n-1})$ is not well-defined.
In this case, the algorithm uses the average rating of product $\nu_n$ in the training data to generate the prediction:
$$\tx_{\nu_n,x^{n-1},W} = \min \left\{ \osMax, \max \left\{ \osMin,
\frac{\sum_{\{w:w \in W,w_{\nu_n} \not= ?\}} w_{\nu_n}}{|\{w:w \in W, w_{\nu_n} \not= ?\}|} \right\} \right\}.$$

In our experiments, we tuned the number of neighbors $k$ by cross validation over the range
$\Gamma=\{1,2,\ldots,40\}$ and settled at $k=10$.

Note that even though the $k$NN algorithm generates scalar predictions, it still fits our definition of
CF algorithms because it is possible to come up with PMFs whose corresponding expectations
equal the predictions $\tx_{\nu_n,x^{n-1},W}$. We do not explicitly define such a PMF, however,
because it is not necessary for computing the empirical RMS distortion in our experiments.

\subsection{Results}\label{su:results}

Figure \ref{fig:distortions} shows the empirical RMS distortions for the three algorithms
that we tested, with different fractions of manipulated data.
Our results suggest that in practice, NB and KDE algorithms are significantly more robust than $k$NN.
In particular, when a user's ratings history is short, $k$NN and NB both incur higher empirical RMS distortions
than KDE. This difference arises because while $k$NN and NB ignore question marks, KDE uses them
and as a result, tempers its predictions. To gain some intuition, let
us consider the following problem instance where ratings are binary: the set of honest ratings $Y$ consists of
$K$ vectors whose entries are all $1$s and as many vectors whose entries are all question marks.
The set of manipulated ratings $Z$ consists of $K$ vectors whose entries are all $0$s and as many
vectors whose entries are all question marks.
To predict the first rating $\ox_{\nu_1}$ of an active user, $k$NN would yield a prediction of $1$ and $1/2$ based
on $Y$ and $(Y,Z)$, respectively, incurring an RMS distortion of $1/2$.
KDE would yield a prediction close to $3/4$ based on $Y$ and a prediction of
$1/2$ based on $(Y,Z)$, incurring an RMS distortion of $1/4$, significantly
less than that of $k$NN. Clearly, the presence of question marks smooths KDE's predictions and
keeps its distortion low.

In Figure \ref{fig:distortions}, as more ratings are provided, distortions incurred by all three algorithms decrease.
When a user's ratings history is long, NB and KDE incur distortions
significantly lower than that of $k$NN. Note that distortions of NB and KDE always stay below the bound in
Corollary \ref{co:error}.
The curves for $k$NN are flat for $n \le 2$ because the algorithm provides the same predictions
for the first two ratings $\ox_{\nu_1}$ and $\ox_{\nu_2}$ of an active user. Note that as fraction of manipulated data $r$
increases, distortions incurred by all three algorithms increase as well.


\begin{figure}[htpb!]
  \centering
  \subfloat[$r=0.1$]{
  \begin{tikzpicture}[x=0.11in,y=8.1in]
    \node[coordinate] at (1,0) (origin) {};
    \node[coordinate] at (40,0) (x) {};
    \node[coordinate] at (1,0.25) (y) {};
    \draw[red,thick] plot file {data/bnd_0.1.table};
    \draw[brown,thick] plot file {data/kNN_distortions_0.1.table};
    \draw[blue,thick] plot file {data/NB_distortions_0.1.table};
    \draw[green,thick] plot file {data/KDE_distortions_0.1.table};
    \draw (1,0) node[left,yshift=2pt] {$0$};
    \foreach \y in {0.05,0.1,0.15,0.2,0.25}
    \draw (1,\y) node[left] {$\y$};
    \draw (1,0.25) node[yshift=14pt] {$\hd_n(p, \nu^X, X, Y, Z)$};
    \foreach \y in {0.05,0.1,0.15,0.2,0.25}
    {
      \draw[shift={(1,\y -| y)}] (5pt,0pt) -- (0pt,0pt);
      \draw[shift={(1,\y -| x)}] (-5pt,0pt) -- (0pt,0pt);
    }
    \foreach \y in {0.05,0.1,0.15,0.2,0.25}
    {
      \foreach \z in {0.01,0.02,0.03,0.04} {
        \draw[shift={(1,\y -| y)},shift={(0,-0.05)},shift={(0,\z)}]
        (2.5pt,0pt) -- (0pt,0pt);
        \draw[shift={(1,\y -| x)},shift={(0,-0.05)},shift={(0,\z)}]
        (-2.5pt,0pt) -- (0pt,0pt);
      }
    }
    \draw (1,0) node[below,xshift=2pt] {$1$};
    \foreach \x in {5, 10, 15, 20, 25, 30, 35, 40}
    \draw (\x,0) node[below] {$\x$};
    \draw (40,0) node[below,xshift=15pt] { $n$};
    \foreach \x in {5, 10, 15, 20, 25, 30, 35, 40}
    \draw[shift={(\x,0)}] (0pt,5pt) -- (0pt,0pt);
    \draw[thick] (origin) -- (x)
    -- (y -| x) -- (y) -- (origin);
    \draw (3,0.13) node[right] {\small Bound};
    \draw (3,0.086) node[right] {\small $k$NN};
    \draw (3,0.049) node[right] {\small NB};
    \draw (3,0.029) node[right] {\small KDE};
  \end{tikzpicture}
  }

  \subfloat[$r=0.3$]{
  \begin{tikzpicture}[x=0.11in,y=8.1in]
    \node[coordinate] at (1,0) (origin) {};
    \node[coordinate] at (40,0) (x) {};
    \node[coordinate] at (1,0.25) (y) {};
    \draw[red,thick] plot file {data/bnd_0.3.table};
    \draw[brown,thick] plot file {data/kNN_distortions_0.3.table};
    \draw[blue,thick] plot file {data/NB_distortions_0.3.table};
    \draw[green,thick] plot file {data/KDE_distortions_0.3.table};
    \draw (1,0) node[left,yshift=2pt] {$0$};
    \foreach \y in {0.05,0.1,0.15,0.2,0.25}
    \draw (1,\y) node[left] {$\y$};
    \draw (1,0.25) node[yshift=14pt] {$\hd_n(p, \nu^X, X, Y, Z)$};
    \foreach \y in {0.05,0.1,0.15,0.2,0.25}
    {
      \draw[shift={(1,\y -| y)}] (5pt,0pt) -- (0pt,0pt);
      \draw[shift={(1,\y -| x)}] (-5pt,0pt) -- (0pt,0pt);
    }
    \foreach \y in {0.05,0.1,0.15,0.2,0.25}
    {
      \foreach \z in {0.01,0.02,0.03,0.04} {
        \draw[shift={(1,\y -| y)},shift={(0,-0.05)},shift={(0,\z)}]
        (2.5pt,0pt) -- (0pt,0pt);
        \draw[shift={(1,\y -| x)},shift={(0,-0.05)},shift={(0,\z)}]
        (-2.5pt,0pt) -- (0pt,0pt);
      }
    }
    \draw (1,0) node[below,xshift=2pt] {$1$};
    \foreach \x in {5, 10, 15, 20, 25, 30, 35, 40}
    \draw (\x,0) node[below] {$\x$};
    \draw (40,0) node[below,xshift=15pt] { $n$};
    \foreach \x in {5, 10, 15, 20, 25, 30, 35, 40}
    \draw[shift={(\x,0)}] (0pt,5pt) -- (0pt,0pt);
    \draw[thick] (origin) -- (x)
    -- (y -| x) -- (y) -- (origin);
    \draw (7,0.162) node[right] {\small Bound};
    \draw (7,0.132) node[right] {\small $k$NN};
    \draw (7,0.096) node[right] {\small NB};
    \draw (7,0.051) node[right] {\small KDE};
  \end{tikzpicture}
  }

  \subfloat[$r=0.5$]{
  \begin{tikzpicture}[x=0.11in,y=8.1in]
    \node[coordinate] at (1,0) (origin) {};
    \node[coordinate] at (40,0) (x) {};
    \node[coordinate] at (1,0.25) (y) {};
    \draw[red,thick] plot file {data/bnd_0.5.table};
    \draw[brown,thick] plot file {data/kNN_distortions_0.5.table};
    \draw[blue,thick] plot file {data/NB_distortions_0.5.table};
    \draw[green,thick] plot file {data/KDE_distortions_0.5.table};
    \draw (1,0) node[left,yshift=2pt] {$0$};
    \foreach \y in {0.05,0.1,0.15,0.2,0.25}
    \draw (1,\y) node[left] {$\y$};
    \draw (1,0.25) node[yshift=14pt] {$\hd_n(p, \nu^X, X, Y, Z)$};
    \foreach \y in {0.05,0.1,0.15,0.2,0.25}
    {
      \draw[shift={(1,\y -| y)}] (5pt,0pt) -- (0pt,0pt);
      \draw[shift={(1,\y -| x)}] (-5pt,0pt) -- (0pt,0pt);
    }
    \foreach \y in {0.05,0.1,0.15,0.2,0.25}
    {
      \foreach \z in {0.01,0.02,0.03,0.04} {
        \draw[shift={(1,\y -| y)},shift={(0,-0.05)},shift={(0,\z)}]
        (2.5pt,0pt) -- (0pt,0pt);
        \draw[shift={(1,\y -| x)},shift={(0,-0.05)},shift={(0,\z)}]
        (-2.5pt,0pt) -- (0pt,0pt);
      }
    }
    \draw (1,0) node[below,xshift=2pt] {$1$};
    \foreach \x in {5, 10, 15, 20, 25, 30, 35, 40}
    \draw (\x,0) node[below] {$\x$};
    \draw (40,0) node[below,xshift=15pt] { $n$};
    \foreach \x in {5, 10, 15, 20, 25, 30, 35, 40}
    \draw[shift={(\x,0)}] (0pt,5pt) -- (0pt,0pt);
    \draw[thick] (origin) -- (x)
    -- (y -| x) -- (y) -- (origin);
    \draw (7,0.23) node[right] {\small Bound};
    \draw (7,0.16) node[right] {\small $k$NN};
    \draw (7,0.12) node[right] {\small NB};
    \draw (7,0.08) node[right] {\small KDE};
  \end{tikzpicture}
  }
  \caption{Empirical RMS distortion as a function of $n$, for different $r$.\label{fig:distortions}}
\end{figure}
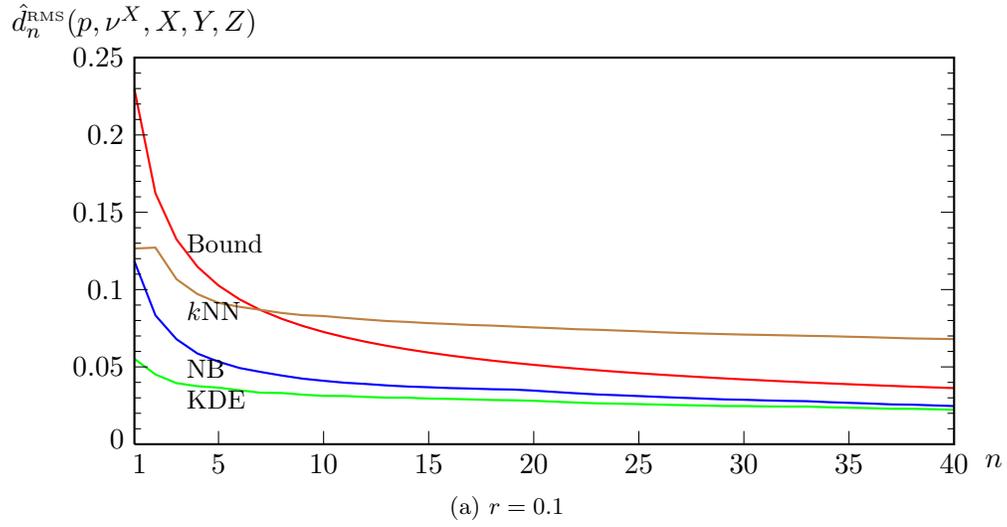
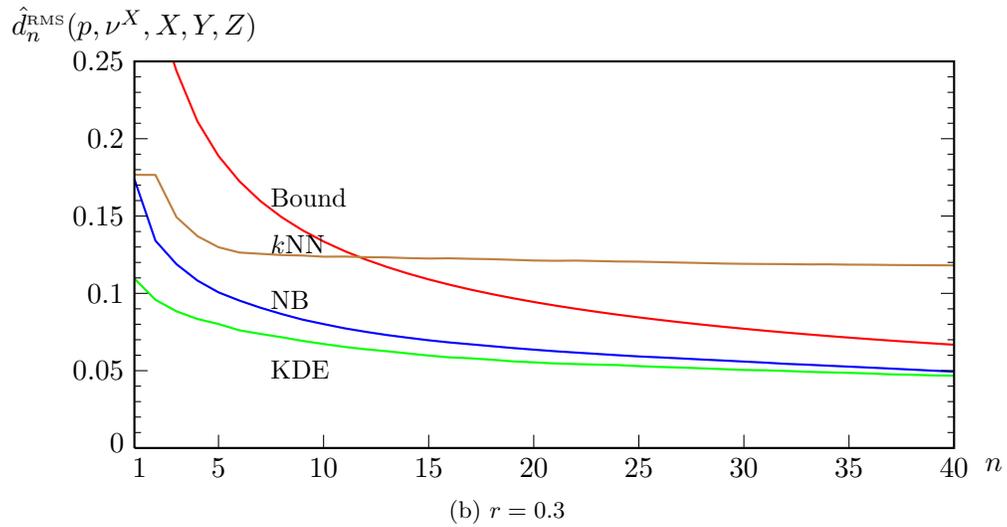
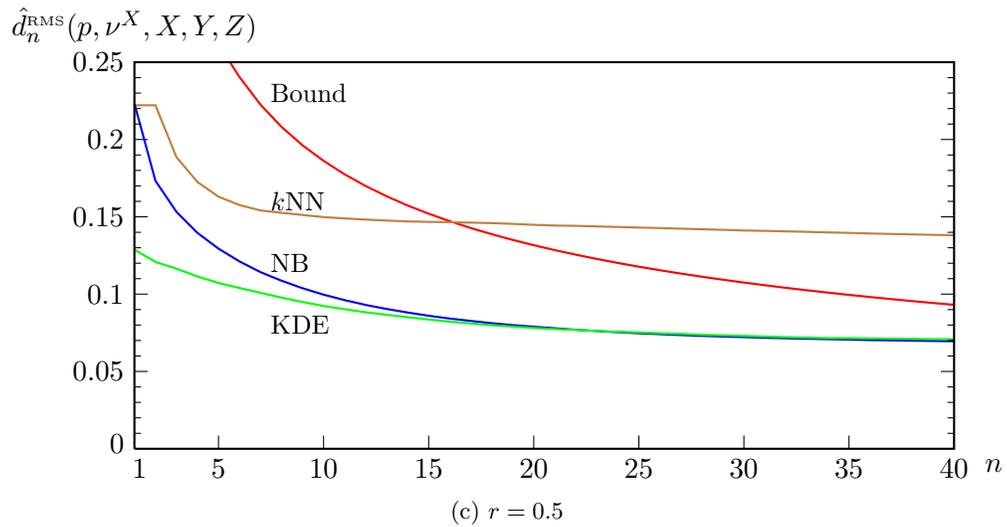


Figure \ref{fig:predictions} displays the empirical RMS prediction errors of the three algorithms.
When $n$ is large, their errors all decrease and in particular, NB offers the lowest error and $k$NN, the highest.
$k$NN sees a spike around $n=3$ because the algorithm
switches its prediction method there: it generates predictions by using average ratings of
all users for $n \le 2$ and generates predictions by using average ratings of neighbors for $n \ge 3$.


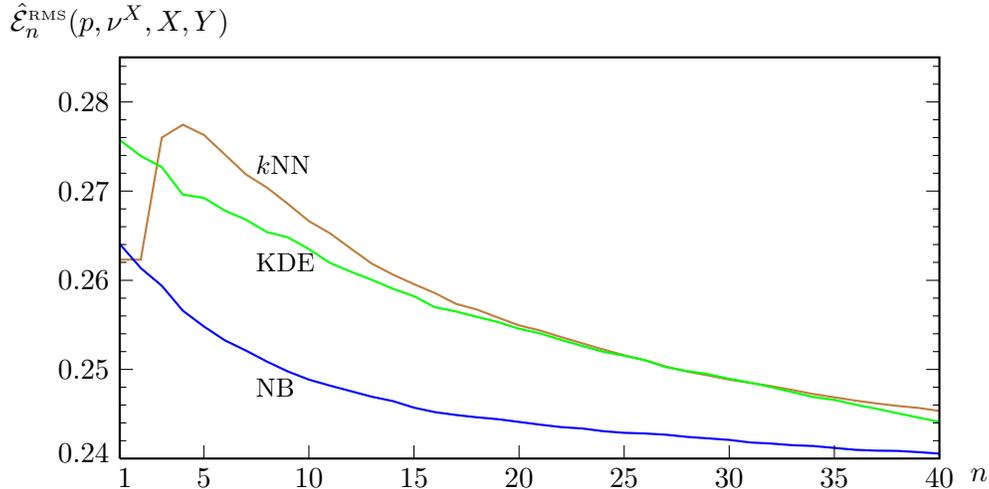
\begin{figure}[htpb!]
  \centering
  \begin{tikzpicture}[x=0.11in,y=46.7in]
    \node[coordinate] at (1,.24) (origin) {};
    \node[coordinate] at (40,.24) (x) {};
    \node[coordinate] at (1,.285) (y) {};
    \draw[brown,thick] plot file {data/kNN_preds.table};
    \draw[blue,thick] plot file {data/NB_preds.table};
    \draw[green,thick] plot file {data/KDE_preds.table};
    \draw (1,.24) node[left,yshift=2pt] {$0.24$};
    \foreach \y in {0.25,0.26,0.27,0.28}
    \draw (1,\y) node[left] {$\y$};
    \draw (1,.285) node[yshift=14pt] {$\hcE_n(p,\nu^X,X,Y)$};
    \foreach \y in {0.25,0.26,0.27,0.28}
    {
      \draw[shift={(1,\y -| y)}] (5pt,0pt) -- (0pt,0pt);
      \draw[shift={(1,\y -| x)}] (-5pt,0pt) -- (0pt,0pt);
    }
    \foreach \y in {0.25,0.26,0.27,0.28}
    {
      \foreach \z in {0.002,0.004,0.006,0.008} {
        \draw[shift={(1,\y -| y)},shift={(0,-0.01)},shift={(0,\z)}]
        (2.5pt,0pt) -- (0pt,0pt);
        \draw[shift={(1,\y -| x)},shift={(0,-0.01)},shift={(0,\z)}]
        (-2.5pt,0pt) -- (0pt,0pt);
      }
    }
    \foreach \y in {0.29}
    {
      \foreach \z in {0.002,0.004} {
        \draw[shift={(1,\y -| y)},shift={(0,-0.01)},shift={(0,\z)}]
        (2.5pt,0pt) -- (0pt,0pt);
        \draw[shift={(1,\y -| x)},shift={(0,-0.01)},shift={(0,\z)}]
        (-2.5pt,0pt) -- (0pt,0pt);
      }
    }
    \draw (1,.24) node[below,xshift=2pt] {$1$};
    \foreach \x in {5, 10, 15, 20, 25, 30, 35, 40}
    \draw (\x,.24) node[below] {$\x$};
    \draw (40,.24) node[below,xshift=15pt] { $n$};
    \foreach \x in {5, 10, 15, 20, 25, 30, 35, 40}
    \draw[shift={(\x,.24)}] (0pt,5pt) -- (0pt,0pt);
    \draw[thick] (origin) -- (x)
    -- (y -| x) -- (y) -- (origin);
    \draw (7,0.273) node[right] {\small $k$NN};
    \draw (7,0.248) node[right] {\small NB};
    \draw (7,0.262) node[right] {\small KDE};
  \end{tikzpicture}
  \caption{Empirical RMS prediction error as a function of $n$.\label{fig:predictions}}
\end{figure}

To get a better sense of our results, we note that Netflix announced that its proprietary algorithm achieves
an empirical RMS prediction error, normalized to our scale, of $0.238$ on a large test set,
and will award one million dollars to
anyone that improves it to $0.214$ \cite{Netflix}.
One might wonder why a decrease of $0.024$ may have such a large impact on
recommendation quality. We suspect that due to the large number of movies,
many of them are given similar predicted ratings. As a result, a small improvement
in prediction accuracy may tease apart these movies and identify the most desirable ones.

Compared to Netflix's benchmark and target prediction errors, our results are reasonable
but not competitive. This is because we did not focus on optimizing the prediction accuracy of the algorithms.
If our objective was to achieve the highest possible accuracy while maintaining
reasonable robustness, one option we could try is to
fine-tune our robust algorithms to be accurate.
For example, for KDE algorithms, we could work to identify more effective kernels.
For NB algorithms, we could choose
different priors or use methods other than expectation-maximization to find the model parameters.
We could probably also design other robust linear and asymptotically linear CF algorithms that achieve higher accuracy
as well.
Overall, we are not suggesting that in practice, the specific algorithms that we presented should be directly implemented.
Instead, one should either use them as starting points or take the insights that they yield into consideration
when designing accurate and robust CF systems.

\section{Conclusion}

Our analytical and empirical work suggests that
linear and asymptotically linear
algorithms can be more robust to manipulation than commonly used nearest neighbor algorithms.
Our results also suggest that it is possible to design algorithms that achieve accuracy alongside robustness.
As such, recommendation systems of Internet commerce sites may improve their robustness to manipulation
by adopting the approaches that we describe. They may also use the bounds on distortion that we establish
as a guide on how many ratings each user should provide to a recommendation system before its predictions
can be trusted.

The simple setting in our work serves as a context for the initial development of our idea,
and can be extended in multiple ways.
One direction is to study the robustness of collaborative filtering algorithms
as measured by alternative metrics. One metric could be, for instance, a user's utility loss
due to manipulation. Another extension is to design algorithms
that provide non-asymptotic guarantees on both prediction accuracy and robustness.

The framework that we establish also facilitates studying the effectiveness of alternative
techniques to abate influence by manipulators. For instance, given a scheme that incentivizes
users to inspect and rate products, one could analyze how honest users and manipulators would
behave, and then use our distortion metrics to assess the robustness of the scheme to manipulation.

It is also worth mentioning that many commercial recommendation systems
build on multiple sources of information, not just collaborative filtering \cite{Adomavicius}.
For example, as discussed in \cite{Balabanovic}, recommendations should also be guided by
features of the products being recommended. An added benefit of the
approaches that we present is that they facilitate coherent fusion of multiple sources of information.

\section*{Acknowledgments}

The authors thank Christina Aperjis, Tom Cover, Paul Cuff, Amir Dembo, Persi Diaconis,
Vivek Farias, John Gill, Ramesh Johari, Yi-Hao Kao, Yi Lu, Taesup Moon, Beomsoo Park, and
Assaf Zeevi for helpful suggestions.
This research was supported in part by the National Science Foundation through grant IIS-0428868.

{\small
\singlespacing
\bibliography{cf}
}

\appendix

\section{Proofs}

\subsection{Relationships Among Distortion Measures}\label{app:distortion_measures}
Propositions \ref{pr:error_bound} and \ref{pr:binary_bound} state relationships between KL, RMS, and
binary prediction distortions. Lemmas \ref{le:mean_bound} and \ref{le:binary_bound} help prove them.

\begin{lemma}\label{le:mean_bound}
Consider two PMFs $p$ and $q$ with support on the same finite set $U \subset [0,1]$. Let $u$ denote
a dummy variable. It holds that
$$\left| \E_p[u] - \E_q[u] \right| \le \frac{1}{2} \| p - q \|_1.$$
\end{lemma}
\begin{proof}
Let $U=\{u_1,\ldots,u_N\}$ and correspondingly, let $p_i=p(u_i)$ and $q_i=q(u_i)$, for $1 \le i \le N$.
Without loss of generality, let $p_1 - q_1 \ge p_2 - q_2 \ge \cdots \ge p_N - q_N$.
There exists $n$ such that $p_n - q_n \ge 0 \ge p_{n+1} - q_{n+1}$. Hence,
$\sum_{i=1}^n |p_i - q_i| = \sum_{i=n + 1}^N |p_i - q_i|$.
We then have
\begin{eqnarray*}
&& \left| \E_p[u] - \E_q[u] \right| = \left| \sum_{i=1}^N u_i (p_i - q_i) \right| = \left| \sum_{i=1}^n u_i (p_i - q_i) + \sum_{i=n+1}^N u_i (p_i - q_i) \right| \\
&\le& \max \left\{ \left| \sum_{i=1}^n u_i (p_i - q_i) \right|, \left| \sum_{i=n+1}^N u_i(p_i-q_i) \right| \right\}
\le \max \left\{ \sum_{i=1}^n \left| p_i - q_i \right|, \sum_{i=n+1}^N \left| p_i-q_i \right| \right\} \\
&=& \frac{1}{2} \sum_{i=1}^N \left| p_i - q_i \right|.
\end{eqnarray*}
\end{proof}

\begin{proposition}\label{pr:error_bound}
Fix the number of products $N$ and let $p$ be a CF algorithm.
Then, for all $M$, $r \in \{0,1/M,\ldots,(M-1)/M\}$, $Y \in S^{N \times (1-r) M}$, $Z \in S^{N \times r M}$, and $\nu \in \sigma_N$,
$$d_n^{\RMS}(p, \nu, Y,Z) \leq \sqrt{\frac{1}{2} d_n^{\KL}(p, \nu, Y,Z)}.$$
\end{proposition}
\begin{proof}
Recall that $\tx_{\nu_k,x^{k-1},Y}$ and $\tx_{\nu_k,x^{k-1},(Y,Z)}$ denote the expected ratings of product $\nu_k$
with respect to PMFs $p_{\nu_k,x^{k-1},Y}$ and $p_{\nu_k,x^{k-1},(Y,Z)}$, respectively. We have
\begin{eqnarray*}
&& d_n^{\RMS}(p, \nu, Y,Z) \\
&=& \sqrt{\frac{1}{n} \sum_{k=1}^n \E\left[\left(\tx_{\nu_k,x^{k-1},Y} - \tx_{\nu_k,x^{k-1},(Y,Z)}\right)^2 \right]} \\
&\le& \sqrt{\frac{1}{n} \sum_{k=1}^n \E\left[\left(\frac{1}{2} \| p_{\nu_k,x^{k-1},Y} - p_{\nu_k,x^{k-1},(Y,Z)} \|_1 \right)^2 \right]} \\
&\le& \sqrt{\frac{1}{2n} \sum_{k=1}^n \E\left[ D\left(p_{\nu_k,x^{k-1},Y} \mid\mid p_{\nu_k,x^{k-1},(Y,Z)} \right) \right]} \\
&=& \sqrt{\frac{1}{2} d_n^{\KL}(p,\nu,Y,Z)},
\end{eqnarray*}
where the first inequality follows from Lemma \ref{le:mean_bound} and
the second inequality follows from Pinsker's inequality.
\end{proof}

\begin{lemma}\label{le:binary_bound}
Consider a Bernoulli random variable $X$ and discrete random variables $W_1$ and $W_2$.
Let $\hX_1$ and $\hX_2$ be the maximum {\it a posteriori} estimates of $X$
upon observing $W_1$ and $W_2$, respectively. That is,
$$\hX_1 = \argmax_{x \in \{0,1\}} \Pr(X=x|W_1),$$
$$\hX_2 = \argmax_{x \in \{0,1\}} \Pr(X=x|W_2).$$
Then,
$$\Pr\left( \hX_1=X \right) - \Pr \left( \hX_2=X \right) \le \sqrt{\E[\left( \E[X|W_1] - \E[X|W_2] \right)^2]}.$$
\end{lemma}
\begin{proof}
\begin{eqnarray*}
&& \Pr\left( \hX_1=X \right) - \Pr \left( \hX_2=X \right) \\
&=& \sum_{w_1} \Pr(W_1=w_1) \max_{x \in \{0,1\}} \Pr(X=x|W_1=w_1) - \sum_{w_2} \Pr(W_2=w_2) \max_{x \in \{0,1\}} \Pr(X=x|W_2=w_2) \\
&=& \sum_{w_1,w_2}\Pr(W_1=w_1,W_2=w_2)\left( \max_x \Pr(X=x|W_1=w_1) - \max_x \Pr(X=x|W_2=w_2) \right) \\
&\le& \sum_{w_1,w_2} \Pr(W_1=w_1,W_2=w_2) \left|\Pr(X=1|W_1=w_1) - \Pr(X=1|W_2=w_2)\right| \\
&\le& \sqrt{ \sum_{w_1,w_2} \Pr(W_1=w_1,W_2=w_2) \left(\Pr(X=1|W_1=w_1) - \Pr(X=1|W_2=w_2)\right)^2} \\
&=& \sqrt{\E \left[\left( \E[X|W_1] - \E[X|W_2] \right)^2\right]}.
\end{eqnarray*}
The first inequality follows from a simple arithmetic argument, and the second inequality
follows from Jensen's inequality.
\end{proof}

\begin{proposition}\label{pr:binary_bound}
Fix the number of products $N$. Let $p$ be a CF algorithm and $\oS=\{0,1\}$.
Then, for all $M$, $r \in \{0,1/M,\ldots,(M-1)/M\}$, $Y \in S^{N \times (1-r) M}$, $Z \in S^{N \times r M}$, and $\nu \in \sigma_N$,
$$d_n^{\B}(p,\nu,Y,Z) \le d_n^{\RMS}(p, \nu, Y, Z).$$
\end{proposition}
\begin{proof}
Recall that $\cx_{\nu_k,x^{k-1},Y}$ and $\cx_{\nu_k,x^{k-1},(Y,Z)}$ denote the binary predictions on product $\nu_k$
with respect to PMFs $p_{\nu_k,x^{k-1},Y}$ and $p_{\nu_k,x^{k-1},(Y,Z)}$, respectively. We have
\begin{eqnarray*}
&& d_n^{\B}(p,\nu,Y,Z) \\
&=& \frac{1}{n} \sum_{k=1}^n \left(\Pr \left(\ox_{\nu_k} = \cx_{\nu_k,x^{k-1},Y} \right)
 - \Pr \left(\ox_{\nu_k} = \cx_{\nu_k,x^{k-1},(Y,Z)}\right) \right) \\
&\le& \sqrt{ \frac{1}{n} \sum_{k=1}^n \left( \Pr \left(\ox_{\nu_k} = \cx_{\nu_k,x^{k-1},Y} \right)
 - \Pr \left(\ox_{\nu_k} = \cx_{\nu_k,x^{k-1},(Y,Z)}\right) \right)^2} \\
&\le& \sqrt{\frac{1}{n} \sum_{k=1}^n \E\left[\left(\tx_{\nu_k,x^{k-1},Y} - \tx_{\nu_k,x^{k-1},(Y,Z)}\right)^2 \right]} \\
&=&  d_n^{\RMS}(p, \nu, Y, Z).
\end{eqnarray*}
The first inequality follows from Jensen's inequality and
the second inequality follows from Lemma \ref{le:binary_bound}.
\end{proof}

\subsection{Results for Linear Collaborative Filtering Algorithms}\label{app:linear}
Proposition \ref{th:distortion} provides a distortion bound for linear CF algorithms. Lemmas
\ref{le:probabilistic_distortion} and \ref{le:linear_divergence} help prove it.

\begin{lemma}\label{le:probabilistic_distortion}
Fix the number of products $N$ and let $p$ be a probabilistic CF algorithm.
Then, for all $M$, $r \in \{0,1/M,\ldots,(M-1)/M\}$, $Y \in S^{N \times (1-r) M}$, $Z \in S^{N \times r M}$, and $\nu \in \sigma_N$,
\begin{equation*}
d_n^{\KL}(p,\nu, Y, Z) \le \frac{1}{n} D\left(\typeDH^{p,Y} \mid\mid \typeDH^{p,(Y,Z)} \right).
\end{equation*}
\end{lemma}
\begin{proof}
We denote by $\typeDH^{p,W,n}(\cdot | x)$ the conditional PMF of $\ox_n$ conditioned on $x$ based on $W$. We have
\begin{eqnarray*}
&& d_n^{\KL}(p, \nu, Y,Z)\\
&=& \frac{1}{n} \sum_{k=1}^n \E\left[D\left(p_{\nu_k,x^{k-1},Y} \mid\mid p_{\nu_k,x^{k-1},(Y,Z)}\right)\right] \\
&=& \frac{1}{n} \sum_{k=1}^n \E\left[D\left(\typeDH^{p,Y,\nu_k}\left(\cdot|x^{k-1}\right) \mid\mid \typeDH^{p,(Y,Z),\nu_k}\left(\cdot|x^{k-1}\right) \right)\right] \\
&\le& \frac{1}{n} \sum_{k=1}^N \E\left[D\left(\typeDH^{p,Y,\nu_k} \left(\cdot|x^{k-1}\right) \mid\mid \typeDH^{p,(Y,Z),\nu_k} \left(\cdot|x^{k-1}\right) \right)\right] \\
&=& \frac{1}{n} D\left(\typeDH^{p,Y}\mid\mid\typeDH^{p,(Y,Z)}\right).
\end{eqnarray*}
The last equality follows from the chain rule of KL divergence.
\end{proof}

\begin{lemma}\label{le:linear_divergence}
Fix the number of products $N$ and let $p$ be a linear CF algorithm.
Then, for all $M$, $r \in \{0,1/M,\ldots,(M-1)/M\}$, $Y \in S^{N \times (1-r) M}$, $Z \in S^{N \times r M}$, and $\nu \in \sigma_N$,
\begin{equation*}
D\left(\typeDH^{p,Y} \mid\mid \typeDH^{p,(Y,Z)} \right) \le \ln\frac{1}{1-r}.
\end{equation*}
\end{lemma}
\begin{proof}
For any $\ox \in \oS^N$, since $p$ is linear, we have
$$\frac{\typeDH^{p,Y}(\ox)}{\typeDH^{p,(Y,Z)}(\ox)} = \frac{\typeDH^{p,Y}(\ox)}{(1-r)\typeDH^{p,Y}(\ox) + r\typeDH^{p,Z}(\ox)} \le \frac{1}{1-r}.$$
Then,
$$D\left(\typeDH^{p,Y} \mid\mid \typeDH^{p,(Y,Z)}\right) = \sum_{\ox \in \oS^N} \typeDH^{p,Y}(\ox) \ln \frac{\typeDH^{p,Y}(\ox)}{\typeDH^{p,(Y,Z)}(\ox)}
\le  \sum_{\ox \in \oS^N} \typeDH^{p,Y}(\ox) \ln\frac{1}{1-r} = \ln\frac{1}{1-r}.$$
\end{proof}

\begin{th:distortion}
Fix the number of products $N$ and let $p$ be a linear CF algorithm.
Then, for all $M$, $r \in \{0,1/M,\ldots,(M-1)/M\}$, $Y \in S^{N \times (1-r) M}$, $Z \in S^{N \times r M}$, and $\nu \in \sigma_N$,
\begin{equation*}
d_n^{\KL}(p,\nu, Y, Z) \le \frac{1}{n} \ln \frac{1}{1-r}.
\end{equation*}
\end{th:distortion}
\begin{proof}
From Lemmas \ref{le:probabilistic_distortion} and \ref{le:linear_divergence}, we have
$$d_n^{\KL}(p,\nu, Y, Z) \le \frac{1}{n} D\left(\typeDH^{p,Y} \mid\mid \typeDH^{p,(Y,Z)} \right) \le \frac{1}{n}\ln\frac{1}{1-r}.$$
\end{proof}

\subsection{Results for Asymptotically Linear Collaborative Filtering Algorithms}\label{app:asymptotically_linear}
Theorem \ref{th:asymptotic_distortion} provides a distortion bound for asymptotically linear CF algorithms.
Lemmas \ref{le:abs_KL_convergence} and \ref{le:KL} help prove it.
Theorem \ref{th:consistent} states the relationship between consistent and asymptotically CF algorithms.
Lemmas \ref{le:compact}, \ref{le:KL_convergence}, and \ref{le:convex_KL_convergence} help prove it.

\begin{lemma}\label{le:abs_KL_convergence}
Let $\{\mu_m\}$ and $\{\nu_m\}$ be two sequences of random PMFs over a fixed finite sample space $\Omega$.
If for all $\epsilon>0$, $\Pr \left( D \left( \mu_m \mid\mid \nu_m \right) \ge \epsilon \right) \rightarrow 0$,
then for all $\epsilon>0$,
$$\Pr \left( \sum_{\omega \in \Omega} \mu_m(\omega) \left| \log \frac{\mu_m(\omega)}{\nu_m(\omega)} \right| \ge \epsilon \right) \rightarrow 0.$$
\end{lemma}
\begin{proof}
For $\epsilon>0$ and $m$, we denote by $A_{m,\epsilon}$ the event that for all $\omega \in \Omega$, at least one
of the following holds:
$\left| \log \left( \mu_m(\omega) / \nu_m(\omega) \right) \right| \le \epsilon$ and
$\max \left\{ \mu_m(\omega), \nu_m(\omega) \right\} \le \epsilon$.
We now prove that for all $\epsilon>0$, $\Pr \left( A_{m,\epsilon} \right) \rightarrow 1.$
To see this, for any given $\epsilon>0$, we let $\delta_{\epsilon}$ be in
$\left(0, \epsilon \left( 1 - e^{-\epsilon} \right) \right)$ and denote by $B_{m,\delta_{\epsilon}}$ the event that
for all $\omega \in \Omega$, $\left| \mu_m(\omega) - \nu_m(\omega) \right| \le \delta_{\epsilon}$. We now show that
$B_{m,\delta_{\epsilon}} \subset A_{m,\epsilon}$. If $\left| \mu_m(\omega) - \nu_m(\omega) \right| \le \delta_{\epsilon}, \forall \, \omega$,
then for any $\omega'$ such that $\max\{\mu_m(\omega'),\nu_w(\omega')\} > \epsilon$, we have
$\min\{\mu_m(\omega'), \nu_m(\omega')\} > \epsilon - \delta_{\epsilon}$. This implies
\begin{eqnarray*}
&& \left| \log \frac{\mu_m(\omega')}{\nu_m(\omega')} \right|
= \max \left\{ \log \frac{\mu_m(\omega')}{\nu_m(\omega')}, \log \frac{\nu_m(\omega')}{\mu_m(\omega')} \right\} \\
&\le& \max \left\{ \log \left( \frac{\left| \mu_m(\omega') - \nu_m(\omega') \right|}{ \nu_m(\omega') } + 1\right),
\log \left( \frac{\left| \nu_m(\omega') - \mu_m(\omega') \right|}{ \mu_m(\omega') } + 1\right) \right\}
< \log \left( \frac{\delta_{\epsilon}}{\epsilon - \delta_{\epsilon}} + 1 \right)
\le \epsilon,
\end{eqnarray*}
where the last inequality follows from our choice of $\delta_{\epsilon}$. Hence, $B_{m,\delta_{\epsilon}} \subset A_{m,\epsilon}$
and for any $\epsilon$ and a corresponding $\delta_{\epsilon}$, we have $\Pr(A_{m,\epsilon}) \ge \Pr(B_{m,\delta_{\epsilon}}) \rightarrow 1$,
where convergence follows from Pinsker's inequality.

For each $m$ and each realization of $\mu_m$ and $\nu_m$,
we let $\Omega_m = \{\omega \in \Omega: \mu_m(\omega) \le \nu_m(\omega)\}$, let
$$\tau_m = \sum_{\omega \in \Omega_m} \mu_m(\omega) \log \frac{\mu_m(\omega)}{\nu_m(\omega)},$$
and for $\epsilon>0$, denote by $C_{m,\epsilon}$ the event that $\left| \tau_m \right| \le \epsilon$. We now show that for any $\epsilon>0$,
$A_{m,\gamma_{\epsilon}} \subset C_{m,\epsilon}$ where $\gamma_{\epsilon} \in (0, \min\{\epsilon/|\Omega|, 1/e\}]$
satisfies $\gamma_{\epsilon} |\log \gamma_{\epsilon}| \le \epsilon/|\Omega|$. To see this,
we first let
$\Omega_{m,\gamma_{\epsilon}}^1 = \{\omega \in \Omega: \left| \log \left(\mu_m(\omega)/\nu_m(\omega)\right) \right| \le \gamma_{\epsilon} \}$
and $\Omega_{m,\gamma_{\epsilon}}^2 = \{\omega \in \Omega: \max \left\{ \mu_m(\omega), \nu_m(\omega) \right\} \le \gamma_{\epsilon} \} \char92 \Omega_{m,\gamma_{\epsilon}}^1$.
Note that $\Omega_{m,\gamma_{\epsilon}}^1 \cap \Omega_{m,\gamma_{\epsilon}}^2 = \emptyset$.
If for all $\omega \in \Omega$,
$\left| \log \left( \mu_m(\omega) / \nu_m(\omega) \right) \right| \le \gamma_{\epsilon}$ or
$\max \left\{ \mu_m(\omega), \nu_m(\omega) \right\} \le \gamma_{\epsilon}$, then
$\Omega_{m,\gamma_{\epsilon}}^1 \cup \Omega_{m,\gamma_{\epsilon}}^2 = \Omega$.
This implies
\begin{eqnarray*}
&& \left| \tau_m \right| = \sum_{\omega \in \Omega_m \cap \Omega_{m,\gamma_{\epsilon}}^1}
\mu_m(\omega) \left| \log \frac{\mu_m(\omega)}{\nu_m(\omega)} \right|
+ \sum_{\omega \in \Omega_m \cap \Omega_{m,\gamma_{\epsilon}}^2}
\mu_m(\omega) \left| \log \frac{\mu_m(\omega)}{\nu_m(\omega)} \right| \\
&\le& |\Omega_{m,\gamma_{\epsilon}}^1| \gamma_{\epsilon}
+ \sum_{\omega \in \Omega_m \cap \Omega_{m,\gamma_{\epsilon}}^2}
\mu_m(\omega) \left| \log \mu_m(\omega) \right| \\
&\le& |\Omega_{m,\gamma_{\epsilon}}^1| \gamma_{\epsilon} + |\Omega_{m,\gamma_{\epsilon}}^2| \gamma_{\epsilon} |\log \gamma_{\epsilon}| \\
&\le& |\Omega_{m,\gamma_{\epsilon}}^1| \frac{\epsilon}{|\Omega|} + |\Omega_{m,\gamma_{\epsilon}}^2| \frac{\epsilon}{|\Omega|} \\
&=& \epsilon.
\end{eqnarray*}
The first inequality follows from the definitions of $\Omega_{m,\gamma_{\epsilon}}^1$ and $\Omega_m$.
The second inequality follows from the definition of $\Omega_{m,\gamma_{\epsilon}}^2$ and that $\gamma_{\epsilon} \le 1/e$.
The third inequality follows from other constraints on $\gamma_{\epsilon}$.
Hence, for any $\epsilon>0$ and a corresponding $\gamma_{\epsilon}$ satisfying the aforesaid constraints,
$A_{m,\gamma_{\epsilon}} \subset C_{m,\epsilon}$.

Note that for all $m$ and all realizations of $\mu_m$ and $\nu_m$,
$$\sum_{\omega \in \Omega} \mu_m(\omega) \left| \log \frac{ \mu_m(\omega) }{ \nu_m(\omega) } \right| = D(\mu_m \mid\mid \nu_m) - 2\tau_m.$$
Hence, for any $\epsilon>0$, a corresponding $\gamma_{\epsilon}$, and any $m$,
\begin{eqnarray*}
&& \Pr \left( \sum_{\omega \in \Omega} \mu_m(\omega) \left| \log \frac{ \mu_m(\omega) }{ \nu_m(\omega) } \right| \ge 3 \epsilon \right) \\
&=& \Pr \left( D(\mu_m \mid\mid \nu_m) - 2\tau_m \ge 3 \epsilon, A_{m,\gamma_{\epsilon}}^c \right)
+ \Pr \left( D(\mu_m \mid\mid \nu_m) - 2\tau_m \ge 3 \epsilon, A_{m,\gamma_{\epsilon}} \right) \\
&\le& \Pr \left( A_{m,\gamma_{\epsilon}}^c \right)
+ \Pr \left( D(\mu_m \mid\mid \nu_m) - 2\tau_m \ge 3 \epsilon, C_{m,\epsilon} \right) \\
&\le& \Pr \left( A_{m,\gamma_{\epsilon}}^c \right)
+ \Pr \left( D(\mu_m \mid\mid \nu_m) \ge \epsilon \right) \rightarrow 0.
\end{eqnarray*}
Here, $A_{m,\gamma_{\epsilon}}^c$ denotes the complement of $A_{m,\gamma_{\epsilon}}$. The last inequality
follows from the definition of $C_{m,\epsilon}$. Convergence follows from our original assumption and that
$\Pr(A_{m,\gamma_{\epsilon}}) \rightarrow 1$.
\end{proof}

\begin{lemma}\label{le:KL}
Let $\{\mu_m\}$, $\{\nu_m\}$, and $\{\chi_m\}$ be three sequences of random PMFs
over a fixed finite sample space $\Omega$.
Suppose that for all $\epsilon>0$, $\Pr \left( D \left( \mu_m \mid\mid \nu_m \right) \ge \epsilon \right) \rightarrow 0$.
Further, suppose there exists $b>0$ such that for all $m$, realizations of $\chi_m$ and $\mu_m$, and $\omega \in \Omega$,
$\chi_m(\omega) / \mu_m(\omega) \le b$. Then, for all $\epsilon>0$,
$\Pr \left( \left| D \left( \chi_m \mid\mid \mu_m \right) - D \left( \chi_m \mid\mid \nu_m \right) \right| \ge \epsilon \right)
\rightarrow 0$.
\end{lemma}
\begin{proof}
For any $\epsilon>0$ and $m$,
\begin{eqnarray*}
&& \Pr \left( \left| D \left( \chi_m \mid\mid \mu_m \right) - D \left( \chi_m \mid\mid \nu_m \right) \right| \ge \epsilon \right) \\
&=& \Pr \left( \left| \sum_{\omega \in \Omega} \chi_m(\omega) \log \frac{\mu_m(\omega)}{\nu_m(\omega)} \right| \ge \epsilon \right) \\
&\le& \Pr \left( \sum_{\omega \in \Omega} \frac{\chi_m(\omega)}{\mu_m(\omega)} \left| \mu_m(\omega)
\log \frac{\mu_m(\omega)}{\nu_m(\omega)} \right| \ge \epsilon \right) \\
&\le& \Pr \left( b \sum_{\omega \in \Omega} \left| \mu_m(\omega)
\log \frac{\mu_m(\omega)}{\nu_m(\omega)} \right| \ge \epsilon \right) \rightarrow 0,
\end{eqnarray*}
where convergence follows from Lemma \ref{le:abs_KL_convergence}.
\end{proof}

\begin{th:asymptotic_distortion}
Fix the number of products $N$ and a set $\jointDSet$ of joint PMFs over $\oS^N \times S^N$.
Let $p$ be a CF algorithm asymptotically linear with respect to $\jointDSet$.
Then, for all $\hJointD^*, \mJointD^* \in \jointDSet$, $r \in [0,1)$, $\nu \in \sigma_N$,
$n \in \Z_+$, and $\epsilon > 0$,
\begin{equation*}
\lim_{m \rightarrow \infty} \Pr \left( d_n^{\KL}(p,\nu, Y_{m}, Z_{m}) \ge \frac{1}{n} \ln \frac{1}{1-r} + \epsilon \right) = 0,
\end{equation*}
where, for each $m$, $Y_{m} = (y^1,\ldots,y^{m-l}) \in S^{N \times {(m-l)}}$,
$Z_{m} = (z^1,\ldots,z^l) \in S^{N \times l}$, $l \sim \mbox{Binomial}(m,r)$,
and $y^1,\ldots,y^{m-l} \sim \hRatingsD^*$ and $z^1,\ldots,z^l \sim \mRatingsD^*$ are i.i.d. sequences.
\end{th:asymptotic_distortion}
\begin{proof}
Let finite sample space $\Omega=\oS^N$. For each $m$, let random PMFs
$\mu_m=(1-r) \typeDH^{p,Y_m} + r \typeDH^{p,Z_m}$, $\nu_m=\typeDH^{p,(Y_m,Z_m)}$, and $\chi_m=\typeDH^{p,Y_m}$.
By the definition of asymptotically linear CF algorithms, for all $\epsilon>0$,
$\Pr \left( D \left( \mu_m \mid\mid \nu_m \right) \ge \epsilon \right) \rightarrow 0$.
Clearly, for all $m$, realizations of $\chi_m$ and $\mu_m$, and $\os \in \oS^N$,
$\chi_m(\os) / \mu_m(\os) \le 1/(1-r)$. Hence, for all $\epsilon>0$,
\begin{eqnarray*}
&& \Pr \left( d_n^{\KL}(p,\nu, Y_m, Z_m) \ge \frac{1}{n} \ln \frac{1}{1-r} + \epsilon \right) \\
&\le& \Pr \left( D\left( \typeDH^{p,Y_m} \mid\mid \typeDH^{p,(Y_m,Z_m)} \right) \ge \ln \frac{1}{1-r} + n \epsilon \right) \\
&\le& \Pr \left( D\left( \typeDH^{p,Y_m} \mid\mid \typeDH^{p,(Y_m,Z_m)} \right) \ge D\left( \typeDH^{p,Y_m} \mid\mid (1-r) \typeDH^{p,Y_m} + r \typeDH^{p,Z_m} \right) + n \epsilon \right) \\
&\le& \Pr \left( \left| D(\chi_m \mid\mid \mu_m) - D(\chi_m \mid\mid \nu_m) \right| \ge n \epsilon \right) \rightarrow 0.
\end{eqnarray*}
The second inequality holds because
$D\left( \typeDH^{p,Y_m} \mid\mid (1-r) \typeDH^{p,Y_m} + r \typeDH^{p,Z_m} \right) \le \ln (1/(1-r))$
and convergence follows from Lemma \ref{le:KL}.
\end{proof}

\begin{lemma}\label{le:compact}
Let $U \subset \Re^k$ be a compact set.
Consider a fixed vector $u \in U$ and a sequence of random vectors $\{u_m\}$ for which
$\Pr \left( u_m \in U \right) \rightarrow 1$.
For any continuous function $f:U \rightarrow \Re$,
if $\Pr \left(\| u_m-u \|_1 \ge \epsilon\right) \rightarrow 0$ for all $\epsilon>0$,
then $\Pr \left(\left| f(u_m)-f(u) \right| \ge \epsilon\right) \rightarrow 0$ for all $\epsilon>0$.
\end{lemma}
\begin{proof}
Because the continuous function $f$ defined on compact set $U$ is uniformly continuous,
for each $\epsilon>0$, there exists $\delta>0$ such that for any $v,v' \in U$, $\| v-v' \|_1 \le \delta$
implies that $|f(v)-f(v')| \le \epsilon$. Hence, for all $\epsilon>0$,
\begin{eqnarray*}
&& \Pr \left( \left| f(u_m) - f(u) \right| \le \epsilon \right) \\
&\ge& \Pr \left( \left| f(u_m) - f(u) \right| \le \epsilon, u_m \in U \right) \\
&\ge& \Pr \left( \| u_m - u \|_1 \le \delta, u_m \in U \right) \rightarrow 1.
\end{eqnarray*}
\end{proof}

\begin{lemma}\label{le:KL_convergence}
Fix a finite sample space $\Omega$.
Let $\{\mu_m\}$ and $\{\nu_m\}$ be two sequences of random PMFs over $\Omega$
and let $\mu$ be a fixed PMF over $\Omega$.
If for all $\epsilon>0$, $\Pr \left( D \left( \mu_m \mid\mid \mu \right) \ge \epsilon \right) \rightarrow 0$
and $\Pr \left( D \left( \nu_m \mid\mid \mu \right) \ge \epsilon \right) \rightarrow 0$,
then for all $\epsilon>0$, $\Pr \left( D \left( \mu_m \mid\mid \nu_m \right) \ge \epsilon \right) \rightarrow 0$.
\end{lemma}
\begin{proof}
We first identify the support of $\mu$. Without loss of generality, we let
$\mu(\omega_i)=0, \forall \, 1 \le i \le l$ and $\mu(\omega_i) > \delta, \forall \, l < i \le |\Omega|$
for some $l \ge 0$ and $\delta>0$.
In the following, we represent PMFs as vectors in $\Re^{|\Omega|}$. To this end,
we define a set
$T = \{\left(t_1,\ldots,t_{|\Omega|}\right): \sum_i t_i = 1. \, t_i = 0, \forall \, i \le l. \, t_i \ge \delta, \forall \, i > l\}$.
Let compact set $U = T \times T$.
We let $u = (\mu,\mu) \in U$ and define a sequence of random vectors $\{u_m\}$ where each
$u_m = \left( \mu_m, \nu_m \right)$.
Let continuous function $f: U \rightarrow \Re$ be the KL divergence $D(\cdot \mid\mid \cdot)$.
By examining the absolute continuity of $\mu_m$ and $\nu_m$ with respect to $\mu$
and applying Pinsker's inequality, we have $\Pr(u_m \in U) \rightarrow 1$ and
further, for all $\epsilon>0$, $\Pr (\|u_m - u\|_1 \ge \epsilon) \rightarrow 0$. Hence, for all $\epsilon>0$,
by Lemma \ref{le:compact},
$\Pr \left( D \left( \mu_m \mid\mid \nu_m \right) \ge \epsilon \right)
= \Pr \left( \left| D \left( \mu_m \mid\mid \nu_m \right) - D \left(\mu \mid\mid \mu \right) \right| \ge \epsilon \right)
= \Pr \left( \left| f(u_m) - f(u) \right| \ge \epsilon \right) \rightarrow 0$.
\end{proof}

\begin{lemma}\label{le:convex_KL_convergence}
Fix a finite sample space $\Omega$.
Let $\{\mu_m\}$ and $\{\nu_m\}$ be two sequences of random PMFs over $\Omega$
and let $\mu$ and $\nu$ be two fixed PMFs over $\Omega$.
If for all $\epsilon>0$, $\Pr \left( D \left( \mu_m \mid\mid \mu \right) \ge \epsilon \right) \rightarrow 0$
and $\Pr \left( D \left( \nu_m \mid\mid \nu \right) \ge \epsilon \right) \rightarrow 0$,
then for all $r \in [0,1]$ and $\epsilon>0$,
$$\Pr \left( D \left( \left( (1-r) \mu_m + r \nu_m \right) \mid\mid \left( (1-r) \mu + r \nu \right) \right) \ge \epsilon \right) \rightarrow 0.$$
\end{lemma}
\begin{proof}
The proof here is similar to that for Lemma \ref{le:KL_convergence}.
For a fixed $r$, let PMF $\chi = (1-r) \mu + r \nu$. Without loss of generality, we let
$\chi(\omega_i)=0, \forall \, 1 \le i \le l$ and $\chi(\omega_i) > \delta, \forall \, l < i \le |\Omega|$
for some $l \ge 0$ and $\delta>0$.
In the following, we represent PMFs as vectors in $\Re^{|\Omega|}$. To this end,
we define a set
$T = \{\left(t_1,\ldots,t_{|\Omega|}\right): \sum_i t_i = 1. \, t_i = 0, \forall \, i \le l. \, t_i \ge \delta, \forall \, i > l\}$.
Let compact set $U = T \times T$.
We let $u = \left( (1-r) \mu + r \nu, (1-r) \mu + r \nu \right) \in U$ and
define a sequence of random vectors $\{u_m\}$ where each
$u_m = \left( (1-r) \mu_m + r \nu_m, (1-r) \mu + r \nu \right)$.
Let continuous function $f: U \rightarrow \Re$ be the KL divergence $D(\cdot \mid\mid \cdot)$.
By examining the absolute continuity of $\mu_m$ with respect to $\mu$ and that of
$\nu_m$ with respect to $\nu$ and applying Pinsker's inequality, we have $\Pr(u_m \in U) \rightarrow 1$,
and for all $\epsilon>0$, $\Pr (\|u_m - u\|_1 \ge \epsilon) \rightarrow 0$. Hence, for all $\epsilon>0$,
by Lemma \ref{le:compact},
\begin{eqnarray*}
&& \Pr \left( D \left( \left( (1-r) \mu_m + r \nu_m \right) \mid\mid \left( (1-r) \mu + r \nu \right) \right) \ge \epsilon \right) \\
&=& \Pr \left( \left| D \left( \left( (1-r) \mu_m + r \nu_m \right) \mid\mid \left( (1-r) \mu + r \nu \right) \right)
- D \left( \left( (1-r) \mu + r \nu \right) \mid\mid \left( (1-r)\mu + r\nu \right) \right) \right| \ge \epsilon \right) \\
&=& \Pr \left( \left| f(u_m) - f(u) \right| \ge \epsilon \right) \rightarrow 0.
\end{eqnarray*}
\end{proof}

\begin{th:consistent}
Any probabilistic CF algorithm consistent with respect to an identifiable and convex set
$\jointDSet$ is asymptotically linear with respect to $\jointDSet$.
\end{th:consistent}
\begin{proof}
We use the notation in the definition of asymptotically linear CF algorithms in Section \ref{su:asymptotic_linear}
and Lemmas \ref{le:KL_convergence} and \ref{le:convex_KL_convergence}.
Fix an algorithm $p$ consistent with respect to $\jointDSet$.
Let finite sample space $\Omega = \oS^N$. For each $m$, let random PMFs
$\mu_m = \typeDH^{p,U_m}$, $\nu_m = \typeDH^{p,V_m}$, and $\chi_m = \typeDH^{p,(U_m,V_m)}$.
Let $\mu = \typeD$, $\nu = \aTypeD$. Fix $r \in [0,1]$. By the consistency
and convexity of $\jointDSet$, we have for all $\epsilon>0$,
$\Pr \left( D \left( \mu_m \mid\mid \mu \right) \ge \epsilon \right) \rightarrow 0$,
$\Pr \left( D \left( \nu_m \mid\mid \nu \right) \ge \epsilon \right) \rightarrow 0$, and
$\Pr \left( D \left( \chi_m \mid\mid \left( (1-r) \mu + r \nu \right) \right) \ge \epsilon \right) \rightarrow 0$.
By Lemma \ref{le:convex_KL_convergence}, for all $\epsilon>0$,
$\Pr \left( \left( (1-r) \mu_m + r \nu_m \right) \mid\mid \left( (1-r) \mu + r \nu \right) \ge \epsilon \right) \rightarrow 0$.
Then for all $\epsilon>0$,
$$\Pr \left( D\left( \left( (1-r) \typeDH^{p,U_m} + r \typeDH^{p,V_m} \right)
\mid\mid \typeDH^{p,(U_m,V_m)} \right) \ge \epsilon \right)
= \Pr \left( D \left( \left( (1-r) \mu_m + r \nu_m \right) \mid\mid \chi_m \right) \ge \epsilon \right) \rightarrow 0,$$
where convergence follows from Lemma \ref{le:KL_convergence}.
\end{proof}

\subsection{Results for Kernel Density Estimation and Naive Bayes Algorithms}\label{app:kde_nb}
Propositions in this section pertain to KDE and NB algorithms.

\begin{proposition}\label{pr:KDE_linear}
Any KDE algorithm is a linear CF algorithm. Any linear CF algorithm is a KDE algorithm.
\end{proposition}
\begin{proof}
Fix a KDE algorithm $p$. Given $W_1 \in S^{N \times M_1}$ and
$W_2 \in S^{N \times M_2}$, for each $\ox \in \oS^N$,
\begin{eqnarray*}
&& \typeDH^{p,(W_1,W_2)}(\ox) \\
&=& \frac{1}{M_1 + M_2} \left(\sum_{w \in W_1} \Kernel_{w}(\ox)
+ \sum_{w \in W_2} \Kernel_{w}(\ox) \right) \\
&=& \frac{M_1}{M_1 + M_2} \left( \frac{1}{M_1} \sum_{w \in W_1} \Kernel_{w}(\ox) \right)
+ \frac{M_2}{M_1 + M_2} \left( \frac{1}{M_2} \sum_{w \in W_2} \Kernel_{w}(\ox) \right) \\
&=& \frac{M_1}{M_1 + M_2} \typeDH^{p,W_1}(\ox) + \frac{M_2}{M_1 + M_2} \typeDH^{p,W_2}(\ox).
\end{eqnarray*}
\noindent Hence, $p$ is linear.

To show the converse, consider a linear CF algorithm $p'$
that generates $\typeDH^{p',W}$ based on $W$. A KDE algorithm where kernel
$\Kernel_w$ is set equal to $\typeDH^{p',\{w\}}$ for each $w \in W$ is equivalent to $p'$.
\end{proof}

\begin{proposition}\label{pr:nb_consistent}
Fix $q \in [0,1)$. Let $\jointDSet^{q}$ be the set of all joint PMFs $\jointD$ over $\oS^N \times S^N$
for which there exists $L \in \Z_+$,
$\eta \in T_L$, and $\theta_{l,n} \in T_{|\oS|}$ for all $1 \le l \le L, 1 \le n \le N$ such that
for each $(\ow,w) \in \oS^N \times S^N$ where $\ow \rightarrow w$,
$$\jointD(\ow,w) = \left( q^{\| w \|_?} (1-q)^{N - \| w \|_?} \right)
\left( \sum_{l=1}^L \eta_l \prod_{n=1}^N \theta_{l,n}^{(\ow_n)} \right).$$
Then, $\jointDSet^{q}$ is identifiable and convex.
Further, the naive Bayes algorithm is consistent with respect to $\jointDSet^{q}$.
\end{proposition}
\begin{proof}
To show that $\jointDSet^{q}$ is identifiable, consider $\jointD$, $\jointD' \in \jointDSet^{q}$
for which $\ratingsD = \ratingsD'$. We then have for each $(\ow,w) \in \oS^N \times S^N$ where $\ow \rightarrow w$,
\begin{eqnarray*}
&& \jointD(\ow,w) \\
&=& q^{\| w \|_?} (1-q)^{N-\| w \|_?} \typeD(\ow) \\
&=& q^{\| w \|_?} (1-q)^{N-\| w \|_?} \ratingsD(\ow) (1-q)^{-N}\\
&=& q^{\| w \|_?} (1-q)^{N-\| w \|_?} \ratingsD'(\ow) (1-q)^{-N}\\
&=& q^{\| w \|_?} (1-q)^{N-\| w \|_?} \typeD'(\ow) \\
&=& \jointD'(\ow,w).
\end{eqnarray*}
Hence, $\jointD^{q}$ is identifiable.

To show that $\jointDSet^{q}$ is convex,
consider arbitrary PMFs $\jointD$, $\jointD' \in \jointDSet^{q}$.
For each $\lambda \in [0,1]$, their convex combination $\jointD^{\lambda}$ satisfies
\begin{eqnarray*}
&& \jointD^{\lambda}(\ow,w) = \lambda \jointD(\ow,w) + (1-\lambda) \jointD'(\ow,w) \\
&=& q^{\| w \|_?} (1-q)^{N-\| w \|_?}
\left( \sum_{l=1}^L \lambda \eta_l \prod_{n=1}^N \theta_{l,n}^{(\ow_n)} +
\sum_{l=1}^{L'} \lambda {\eta'}_l \prod_{n=1}^N {\theta'}_{l,n}^{(\ow_n)}\right) \\
&=& q^{\| w \|_?} (1-q)^{N-\| w \|_?}
\left( \sum_{l=1}^{L+L'} \heta_l \prod_{n=1}^N \htheta_{l,n}^{(\ow_n)} \right)
\end{eqnarray*}
for each $(\ow,w) \in \oS^N \times S^N$ that $\ow \rightarrow w$,
where $\heta_l = \lambda \eta_l$ for $l \le L$, $\heta_l = (1-\lambda) {\eta'}_{l-L}$ for $l > L$,
and for each $n$, $\htheta_{l,n} = \theta_{l,n}$ for $l \le L$ and $\htheta_{l,n} = \theta_{l-L,n}$ for $l > L$.
Hence, $\jointD^{\lambda} \in \jointDSet^{q}$. $\jointDSet^{q}$ is convex.

We now show that the naive Bayes algorithm $p$ is consistent with respect to $\jointDSet^{q}$
by using results in \cite{Wald}. We will use notation in the definition of consistent CF algorithms
in Section \ref{su:asymptotic_linear}. We also denote by $\typeD^*$ the true PMF over $\oS^N$
and let $\typeDSet = \{\typeD: \jointD \in \jointDSet^q \}$ be the set of all marginals over $\oS^N$
of PMFs in $\jointDSet^q$.
According to Theorem $2$ in \cite{Wald}, if
\begin{enumerate}
\item\label{it:technical} $\oS^N$, $\typeDSet$, and $\typeD^*$ satisfy certain technical conditions
specified in \cite{Wald}, and
\item\label{it:max} There exists a constant $b>0$ such that for all $m$ and all realizations of $\{W_m\}$,
$$\frac{ \prod_{w \in W_m} \typeDH^{p,W_m}(w) }{ \prod_{w \in W_m} \typeD^*(w) } \ge b,$$
\end{enumerate}
then
\begin{equation}\label{eq:oneNormConv}
\| \typeDH^{p,W_m} - \typeD^* \|_1 \rightarrow 0, \, a.s.
\end{equation}
We verify that condition \ref{it:technical} holds in our problem instance and desire to find a
$b$ that satisfies condition \ref{it:max}.
To do so, we denote by $(L^*,\eta^*,\theta^*)$ the parameters corresponding to $\typeD^*$ and
by $(\hL^{W_m}, \heta^{W_m}, \htheta^{W_m})$ the parameters corresponding to $\typeDH^{p,W_m}$.
Recall that we tune the parameter $\tau$ by cross validation over a range $\Gamma$. Let $\tau^*$
be the value that we settle at.
We let
$$b = \left( \min_{\tau \in \Gamma} p_L^{\tau}(L^*) \right) \left( f_{\eta}^L(\eta^*) \prod_{l,n} f_{\theta} (\theta_{l,n}^*) \right).$$
Because $\typeDH^{p,W_m}$ maximizes the posterior PDF,
for all $m$ and all realizations of $\{W_m\}$, we have
$$\frac{ \prod_{w \in W_m} \typeDH^{p,W_m}(w) }{ \prod_{w \in W_m} \typeD^*(w) } \ge
\frac{p_L^{\tau^*}(L^*) f_{\eta}^L(\eta^*) \prod_{l,n} f_{\theta} (\theta_{l,n}^*)}
{p_L^{\tau^*}(\hL^{W_m}) f_{\eta}^L(\heta^{W_m}) \prod_{l,n} f_{\theta} (\htheta_{l,n}^{W_m})} \ge b.$$
Hence, condition \ref{it:max} holds, establishing (\ref{eq:oneNormConv}).
By the continuity of KL divergence and properties of almost sure convergence, for all $\epsilon>0$, we have
$$\Pr \left( D\left( \typeDH^{p,W_m} \mid\mid \typeD^* \right) \ge \epsilon \right) \rightarrow 0,$$
which implies that $p$ is consistent with respect to $\jointDSet^q$.
\end{proof}

\section{Table of Notation}\label{app:notation}
\begin{table}[h!]
\begin{center}
\begin{tabular}{|r|p{14cm}|}
\hline
Notation & Description \\
\hline
$\oS$ & Set of possible rating values. \\
$S$ & Union of $\oS$ and the singleton set containing the question mark. \\
$N$ & Number of products. \\
$Y$ & Set of ratings vectors provided by honest users. \\
$Z$ & Set of ratings vectors provided by manipulators. \\
$W$ & Training set consisting of all ratings vectors. \\
$r$ & Fraction of ratings vectors generated by manipulators. \\
$M$ & Number of ratings vectors in the training set $W$. \\
$n$ & Number of ratings provided by an active user. \\
$x^k$ & Ratings vector in $S^N$ that contains $k$ ratings provided by an active user. \\
$\nu_k$ & The index of a $k$th inspected product by an active user. \\
$p_{\nu_k,x^{k-1},W}$ & {\footnotesize PMF generated on the rating for product $\nu_k$, for a user
with history $x^{k-1}$, based on training set $W$.} \\
$p$ & A CF algorithm. \\
$\sigma_N$ & Set of permutations of $\{1,\ldots,N\}$. \\
$d^{\KL}$ & KL distortion. \\
$d^{\RMS}$ & RMS distortion. \\
$d^{\B}$ & Binary prediction distortion. \\
$\tx_{\nu_k,x^{k-1},W}$ & {\footnotesize Scalar prediction of rating of product $\nu_k$ for a user
with history $x^{k-1}$, based on training set $W$.} \\
$\cx_{\nu_k,x^{k-1},W}$ & {\footnotesize Binary prediction of rating of product $\nu_k$ for a user
with history $x^{k-1}$, based on training set $W$.} \\
$\ow^m$ & $m$th user type in $W$. \\
$w^m$ & $m$th ratings vector in $W$. \\
$\ow_n$ or $\ox_n$ & Rating of an $n$th product based on a user type. \\
$w_n$ or $x_n$ & Observed rating of an $n$th product. \\
$\ow \rightarrow w$ & For each $n$, either $w_n=\ow_n$ or $w_n=?$. \\
$\| w \|_?$ & $|\{n:w_{n}=?\}|$. \\
$\jointD$ & Joint PMF over $\oS^N \times S^N$. \\
$\typeD$ & Marginal PMF over $\oS^N$. \\
$\ratingsD$ & Marginal PMF over $S^N$. \\
$\hJointD^*$ & Joint PMF over $\oS^N \times S^N$ of honest users. \\
$\mJointD^*$ & Joint PMF of $\oS^N \times S^N$ of manipulators. \\
$\jointDSet$ & Set of joint PMFs over $\oS^N \times S^N$ of interest. \\
$\typeDSet$ & Set of marginal PMFs over $\oS^N$ of interest. \\
$\typeDH^{p,W}$ & PMF for $\ox$ generated by algorithm $p$ based on training set $W$. \\
$\typeDH^{p,W}(\cdot | x)$ & {\footnotesize Conditional PMF for $\ox$ conditioned on ratings vector $x \in S^N$,
generated by algorithm $p$ based on training set $W$.} \\
$T_k$ & Simplex $\{(t_1,\ldots,t_k): t_j \ge 0, \forall \, 1 \le j \le k. \sum_{j=1}^k t_j = 1 \}$. \\
$\hd$ & Empirical RMS distortion. \\
$\hcE$ & Empirical RMS prediction error. \\
$X$ & Set of ratings vectors provided by active users. \\
$V$ & Validation set of ratings vectors. \\
$\nu_k^x$ & The index of a $k$th inspected product by an active user with ratings vector $x$. \\
\hline
\end{tabular}
\caption{Table of notation.}
\end{center}
\end{table}
\end{document}